\pdfoutput=1
\documentclass[11pt,letterpaper]{article}

\usepackage[top=1in, bottom=1in, left=1in, right=1in]{geometry}
\usepackage[T1]{fontenc}
\usepackage{times}
\usepackage{helvet}
\usepackage{courier}
\usepackage{amsmath,amssymb,amsthm}
\usepackage{mathtools}
\usepackage{bm}
\usepackage{graphicx}
\usepackage[caption=false,font=footnotesize]{subfig}
\usepackage{algorithm}
\usepackage{algorithmic}
\usepackage{cite}

\usepackage{xcolor}
\usepackage{hyperref}
\usepackage{microtype}
\hypersetup{colorlinks=true,linkcolor=blue,citecolor=blue,urlcolor=blue}

\usepackage{booktabs}
\usepackage{multirow}
\usepackage{enumitem}
\setlist{nosep,leftmargin=*}
\newtheorem{theorem}{Theorem}
\newtheorem{lemma}[theorem]{Lemma}
\newtheorem{corollary}[theorem]{Corollary}
\newtheorem{proposition}[theorem]{Proposition}
\newtheorem{remark}{Remark}
\newtheorem{assumption}{Assumption}
\theoremstyle{definition}
\newtheorem{definition}{Definition}
\newcommand{\R}{\mathbb{R}}
\newcommand{\E}{\mathbb{E}}
\newcommand{\Prob}{\mathbb{P}}
\newcommand{\ones}{\mathbf{1}}
\newcommand{\Fnorm}[1]{\left\| #1 \right\|_F}

\newcommand{\inner}[2]{\left\langle #1,\, #2 \right\rangle}
\newcommand{\tr}{\operatorname{tr}}

\newcommand{\cH}{\mathcal{H}}
\newcommand{\cP}{\mathcal{P}}
\newcommand{\cB}{\mathcal{B}}
\newcommand{\cF}{\mathcal{F}}
\newcommand{\cG}{\mathcal{G}}

\newcommand{\cE}{\mathcal{E}}
\newcommand{\xbar}{\bar{x}}
\newcommand{\ybar}{\bar{y}}
\newcommand{\sbp}{\bar{\sigma}_p^2}
\newcommand{\thp}{\theta_p}

\title{\LARGE \bf Convergence of Byzantine-Resilient Gradient Tracking via Probabilistic Edge Dropout}

\author{%
Amirhossein Dezhboro$^{1,*}$ \qquad Fateme Maleki$^{2}$ \\[1ex]
Arman Adibi$^{3}$ \qquad Erfan Amini$^{4}$ \\[1ex]
Jose E.~Ramirez-Marquez$^{1}$
}

\date{}
\begin{document}
\maketitle

\begingroup
\renewcommand{\thefootnote}{}
\begin{NoHyper}
\footnotetext{\hspace{-1.5em}%
$^{1}$Department of Systems Engineering, Stevens Institute of Technology, Hoboken, NJ 07030.\\[0.5ex]
$^{2}$Department of Industrial and Systems Engineering, Rutgers University, Piscataway, NJ 08854.\\[0.5ex]
$^{3}$School of Computer and Cyber Sciences, Augusta University, Augusta, GA 30912.\\[0.5ex]
$^{4}$Center for Climate Systems Research, Columbia University, Palisades, NY 10964.\\[0.5ex]
$^{*}$Corresponding author: \texttt{adezhbor@stevens.edu}%
}
\end{NoHyper}
\endgroup


\begin{abstract}
We study distributed optimization over networks with Byzantine agents that may send arbitrary adversarial messages. We propose \emph{Gradient Tracking with Probabilistic Edge Dropout} (GT-PD), a stochastic gradient tracking method that preserves the convergence properties of gradient tracking under adversarial communication. GT-PD combines two complementary defense layers: a universal self-centered projection that clips each incoming message to a ball of radius $\tau$ around the receiving agent, and a fully decentralized probabilistic dropout rule driven by a threshold-gated dual-metric trust score in the decision and tracking channels. This design bounds adversarial perturbations while preserving the doubly stochastic mixing structure, a property often lost under robust aggregation in decentralized settings; in particular, we show that the self-centered projection preserves the honest network average \emph{exactly}, for every realization and even when honest messages are clipped. Our convergence analysis follows the linear-system-of-inequalities framework of stochastic gradient tracking: the optimality gap, consensus error, and tracking error obey a coupled recursion $z^{k+1}\le \mathbf{A}z^k+b$, and all guarantees are read off componentwise from $(I-\mathbf{A})^{-1}b$. Under complete Byzantine isolation ($p_b=0$), GT-PD converges linearly to a neighborhood of size $\mathcal{O}\!\big(\tfrac{\alpha\sigma^2}{\mu n_\cH}+\tfrac{\alpha^2\kappa^2\sigma^2}{(1-\bar\sigma_p^2)^4}\big)$, matching centralized stochastic gradient performance to first order in the step size. Under partial isolation ($p_b>0$), we introduce the leaky-tracking variant GT-PD-L and establish linear convergence to a bounded neighborhood governed by the stochastic variance and the clipping-to-leak ratio, under explicit and dimensionally homogeneous step-size and leak conditions, including a joint condition on $\alpha\beta$ that quantifies the cost of leaky integration. Honest non-clipping is established as a high-probability geometric consequence of the dynamics, in both isolation regimes; for $p_b>0$ this requires an explicit budget on the aggregate Byzantine weight, which we show is the natural analogue of the adversary-fraction bounds standard in the Byzantine optimization literature. We further show that under two-tier dropout with $p_h=1$, isolating Byzantine agents introduces no additional variance into the honest consensus dynamics, and we derive a principled rule for the trust-score tolerance $S_0$ from the steady-state disagreement floors. Experiments on MNIST under Sign Flip, ALIE, and Inner Product Manipulation attacks show that GT-PD-L outperforms coordinate-wise trimmed mean by up to 4.3 percentage points under stealth attacks.
\end{abstract}


\section{Introduction}\label{sec:intro}
Motivated by emerging large-scale applications in machine learning, signal processing, and multi-agent systems, we study the problem of distributed optimization over networks in which a subset of agents may behave adversarially. In these settings, $n$ agents collaborate to minimize a global objective $f(x) = \frac{1}{n}\sum_{i=1}^n f_i(x)$ by communicating only with neighbors on a graph $G = (V,E)$. Gradient tracking algorithms \cite{pu2021push, xin2018linear, nedic2017achieving, pu2021distributed} have emerged as the method of choice for achieving exact convergence (as opposed to consensus-based approaches that converge to a neighborhood), by maintaining local gradient estimators that track the global gradient. However, a critical vulnerability arises in the presence of \emph{Byzantine agents}, which are nodes that may send arbitrary, adversarial messages to corrupt the optimization process. The main goal of this work is to develop and analyze a defense mechanism that isolates such adversaries while preserving the favorable convergence properties of gradient tracking.

Byzantine-resilient distributed computation has a rich history rooted in the work of Lamport et al.\ \cite{lamport1982byzantine}, with foundational results on graph-theoretic conditions for resilient consensus established by Sundaram and Hadjicostis \cite{sundaram2011distributed} and LeBlanc et al.\ \cite{leblanc2013resilient}. In the optimization setting, robust aggregation rules such as coordinate-wise trimmed mean \cite{yin2018byzantine} and Krum \cite{blanchard2017machine} have been developed for centralized (parameter server) architectures. However, extending these to fully decentralized networks remains an active and challenging area of research. Recent analyses demonstrate that directly applying centralized robust aggregators often breaks the doubly stochastic property of the virtual mixing matrix, leading to severe disagreement among honest agents \cite{wu2023byzantine}. Despite these architectural hurdles, recent advances have successfully integrated robust aggregation with gradient tracking \cite{li2024efficient} and leveraged local similarity metrics to filter malicious models in decentralized federated learning \cite{fang2024byzantine}.

Nevertheless, existing approaches generally rely on \emph{spatial filtering}, which are deterministic rules that completely discard or modify the mixing weights of suspicious updates at each iteration. A key difficulty is that altering these weights fundamentally breaks the doubly stochastic property of the mixing matrix in decentralized settings, which can lead to biased convergence \cite{wu2023byzantine}. In this paper, we bridge this gap by proposing a dual-layer defense architecture that operates on orthogonal dimensions of the attack space. Rather than filtering the mixing weights, we introduce a \emph{universal self-centered projection} applied to all incoming messages that strictly clips any message deviating beyond a specified radius while leaving the mixing matrix entirely unmodified. Two structural facts drive the analysis. First, because the projection is self-centered and the mixing weights are symmetric, the clipping corrections on honest edges cancel \emph{pairwise in the network average} (Lemma~\ref{lem:honest_clip}): the honest average dynamics are exact for every realization, whether or not honest messages are clipped. Second, under standard initialization, honest messages pass through the projection unmodified with high probability over any fixed horizon (Propositions~\ref{prop:honest_nonclip} and~\ref{prop:honest_nonclip_partial}), so the projection effectively targets only adversarial messages without requiring an explicit detection decision. To manage the temporal dimension, we complement this with \emph{probabilistic edge dropout}, which reduces the frequency of adversarial exposure. By bounding the magnitude through projection and the frequency through dropout, our framework handles arbitrary adversaries while inherently preserving double stochasticity.

\begin{figure}[t]
\centering
\includegraphics[width=0.85\linewidth]{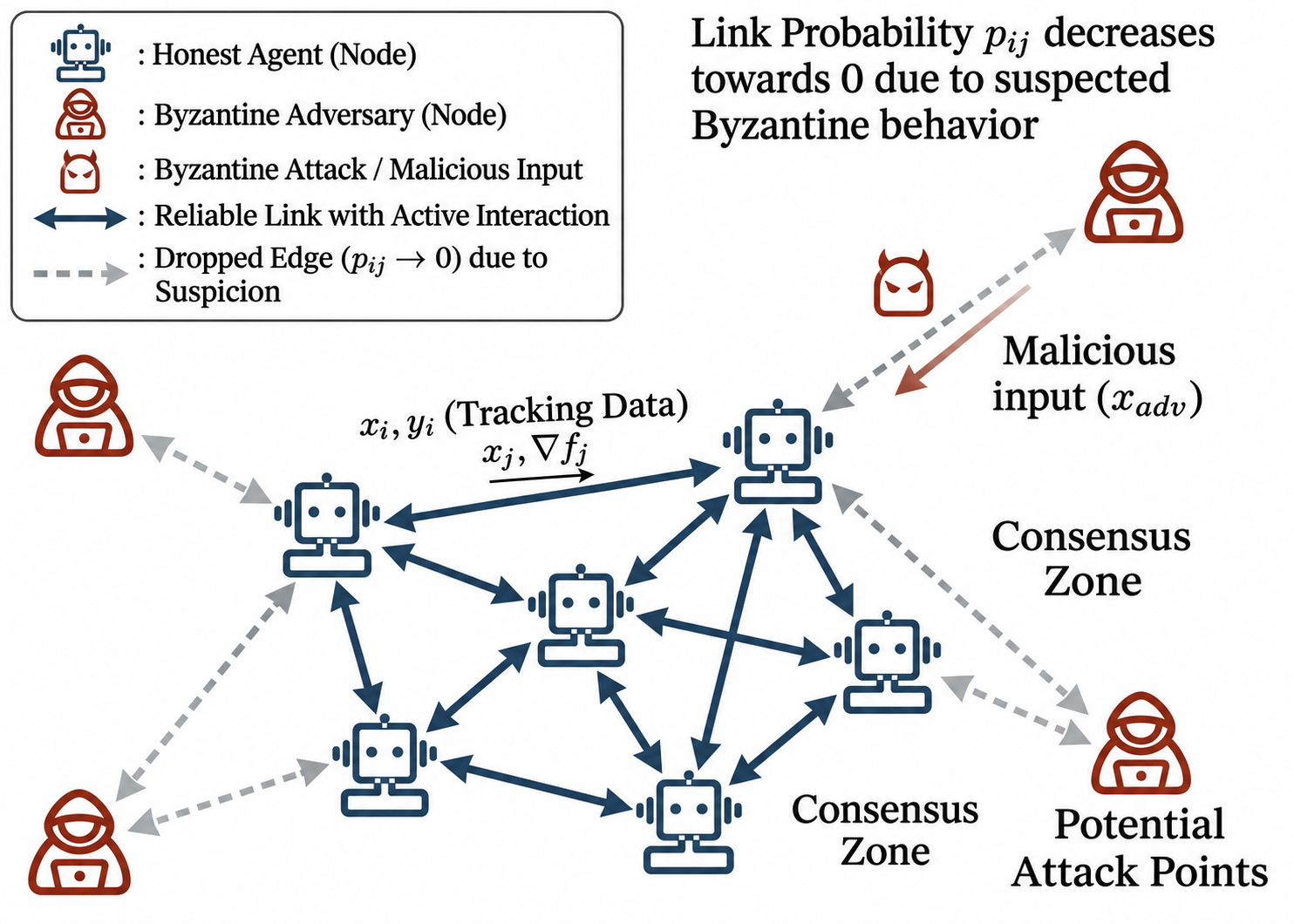}
\caption{GT-PD system model. A decentralized network of $N$ agents communicates over a graph $G$, where $M$ agents are Byzantine adversaries. Honest agents (blue) exchange decision variables via reliable links (solid arrows). All edges are subject to probabilistic dropout based on a local trust score. For edges to suspected adversaries (dashed), the retention probability drops to $p_{ij} \to 0$.}
\label{fig:prob_setting}
\end{figure}

Specifically, we propose \emph{probabilistic edge dropout}. At each iteration, each edge in the communication graph is independently retained with a specified probability and dropped otherwise. Edges connecting to suspected Byzantine agents receive lower retention probabilities, effectively isolating adversarial nodes while preserving connectivity among honest agents. This mechanism is driven by a threshold-gated dual-metric trust score that evaluates neighbor credibility jointly in the spatial (parameter) and temporal (gradient tracking) channels; the retention probability for an edge encodes the trust placed in that neighbor across both channels, preventing adversaries from evading detection by concentrating anomalous behavior in a single channel. The GT-PD system model is illustrated in Figure~\ref{fig:prob_setting}. Critically, this dropout mechanism preserves the doubly stochastic structure of the mixing matrix almost surely (see Proposition~\ref{prop:ds}), a property that spatial filtering approaches generally cannot guarantee.

Our convergence theory deliberately follows the linear-system-of-inequalities template of stochastic gradient tracking \cite{pu2021distributed}: we track the optimality gap, the consensus error, and the tracking error as a vector $z^k$, derive a componentwise recursion $z^{k+1}\le \mathbf{A}z^k+b$, certify $\rho(\mathbf{A})<1$ through a determinant test, and obtain the asymptotic guarantees componentwise from $(I-\mathbf{A})^{-1}b$. This yields neighborhoods that vanish with the step size at the correct orders, together with explicit steady-state levels for the consensus and tracking errors individually. The latter are not a by-product: they are what quantifies the trust-score behavior of honest edges at stationarity (Corollary~\ref{cor:S0_rule}) and what drives the honest non-clipping guarantees. A key structural finding remains: under the standard tracking update, Byzantine perturbations to the gradient tracker accumulate over iterations without contraction. We address this through an exact convergence guarantee under complete isolation ($p_b = 0$), and a modified algorithm (GT-PD-L) with a continuous leaky integrator for $p_b > 0$.

Our contributions are as follows:
\begin{enumerate}[label=(\roman*)]
\item We formalize the GT-PD algorithm, which applies a universal self-centered projection to all incoming messages and independently employs probabilistic edge dropout. The projection bounds message magnitude by construction without requiring neighbor classification (Proposition~\ref{prop:enforced_bound}), and preserves the honest network average exactly even when honest messages are clipped (Lemma~\ref{lem:honest_clip}). We prove that honest messages are never modified by the projection, with high probability over any fixed horizon, under a quantifiable and non-circular condition on the projection radius; under partial isolation this requires an explicit budget on the aggregate Byzantine mixing weight (Propositions~\ref{prop:honest_nonclip} and~\ref{prop:honest_nonclip_partial}). The dropout mixing matrix remains doubly stochastic almost surely (Section~\ref{sec:formulation}).

\item We propose a threshold-gated dual-metric data-driven retention probability (Definition~\ref{def:data_driven_p}) aggregating normalized deviations in both the spatial and gradient tracking channels. The multiplicative factorization $p_{ij}^k = p_{ij}^{k,x} \cdot p_{ij}^{k,y}$ ensures that a deviation in either channel drives the retention probability toward zero, closing the evasion path exploited by decoupled attacks (Proposition~\ref{prop:decoupled_failure}). The threshold gate guarantees that honest edges whose scores fall below the tolerance $S_0$ are retained with probability one, and we establish the universal lower bound $p_{\min} = e^{-\lambda(8-S_0)}$ (Proposition~\ref{prop:p_lower}). We further derive a principled rule for selecting $S_0$ from the steady-state disagreement floors of the algorithm (Corollary~\ref{cor:S0_rule}), explaining the retention statistics observed in our experiments.

\item We introduce the effective honest mixing matrix $\hat{W}^k$ and derive a \emph{quantitative, uniform} second-moment contraction factor,
$\bar\sigma_p^2 \le 1 - p_{\min}^{|\cE_\cH|}\big(1-\|\hat W - J_\cH\|_2^2\big) < 1$,
valid over all iterations and admissible histories, formally isolating the honest subspace from adversarial variance injection (Section~\ref{sec:dropout_properties}).

\item We prove that under complete Byzantine isolation ($p_b = 0$), GT-PD converges linearly to a noise floor determined exclusively by the stochastic gradient variance, with the componentwise steady state
$\limsup_k \E\|\xbar^k-x^*\|^2 = \mathcal{O}\big(\tfrac{\alpha\sigma^2}{\mu n_\cH} + \tfrac{\alpha^2\kappa^2\sigma^2}{(1-\bar\sigma_p^2)^4}\big)$,
under explicit, dimensionally homogeneous step-size conditions (Theorem~\ref{thm:exact}, Corollary~\ref{cor:exact_steady}). The zero-penalty isolation result shows that complete isolation does not slow consensus when $p_h = 1$ (Lemma~\ref{lem:two_tier_spectral}).

\item For partial isolation ($p_b > 0$), we propose GT-PD-L, substituting standard tracking with a continuous leaky integrator. We prove linear convergence to a bounded steady-state neighborhood governed by the stochastic variance and the clipping-to-leak ratio $(1-\beta)^2\zeta^2/\beta^2$, under explicit conditions that include a joint step-size/leak requirement $\alpha\beta \lesssim (1-\bar\sigma_p^2)^2/(\kappa L)$ quantifying the cost of leaky integration (Theorem~\ref{thm:leaky}, Corollary~\ref{cor:leaky_steady}). The optimality-gap recursion and the tracking-invariant bound hold \emph{unconditionally}, by exact average preservation; only the disagreement contractions require the high-probability non-clipping event of Proposition~\ref{prop:honest_nonclip_partial}.
\end{enumerate}

\textbf{Related work.} Gradient tracking over deterministic time-varying graphs was analyzed by Nedic et al.\ \cite{nedic2017achieving}; the stochastic-gradient analysis whose linear-system structure we adopt is due to Pu and Nedi\'c \cite{pu2021distributed}. Random communication topologies arising from link failures were studied for consensus by Hatano and Mesbahi \cite{hatano2005agreement} and Tahbaz-Salehi and Jadbabaie \cite{tahbaz2010consensus}, but not in the context of gradient tracking with intentional dropout. Communication compression in decentralized optimization, studied by Koloskova et al.\ \cite{koloskova2019decentralized}, introduces lossy communication analogous to dropout but does not consider Byzantine agents. While Byzantine-robust decentralized optimization has been studied extensively over both static and time-varying networks \cite{su2016multi, yang2019byrdie, peng2020byzantine, kuwaranancharoen2025geometric}, these approaches generally rely on spatial filtering rather than probabilistic dropout mechanisms, and converge to regions whose size is governed by the heterogeneity of the local objectives rather than by the stochastic variance alone \cite{kuwaranancharoen2025geometric}. Li \cite{li2024efficient} applies robust aggregation (Median, Trimmed Mean, Krum) \emph{within} the gradient tracking loop, handling arbitrary Byzantine messages but breaking double stochasticity.

He et al.\ \cite{he2023clippedgossip} propose ClippedGossip, a self-centered clipping rule for Byzantine-robust decentralized consensus. While GT-PD shares the self-centered projection principle with ClippedGossip, it differs in three key aspects: (i) GT-PD proves that honest messages are geometrically guaranteed to pass through the projection unmodified under a quantifiable condition whose projection radius is computable from problem parameters (Propositions~\ref{prop:honest_nonclip} and~\ref{prop:honest_nonclip_partial}), whereas ClippedGossip clips all messages and absorbs the honest clipping error into the convergence bound, with thresholds that depend on quantities unavailable in a distributed deployment \cite{wu2023byzantine}; (ii) GT-PD integrates the projection with gradient tracking rather than plain consensus; and (iii) GT-PD combines the projection with probabilistic dropout as a complementary frequency-control defense. Explicit bounds on the tolerable adversarial budget are standard across this literature, e.g., the attacked-fraction bound $\rho<1/(1+\kappa)$ of \cite{yu2022secure}, the contraction-constant condition of \cite{wu2023byzantine}, and the $F$-local model of \cite{kuwaranancharoen2025geometric}; the budget on the aggregate Byzantine weight $\delta_{\max}^{\cB}$ required by Proposition~\ref{prop:honest_nonclip_partial} is the analogue for our setting. To the best of our knowledge, GT-PD is the first algorithm that combines universal self-centered projection with probabilistic edge dropout within the gradient tracking framework, simultaneously achieving tolerance to arbitrary Byzantine adversaries, inherent preservation of double stochasticity, and a zero-penalty isolation guarantee under complete dropout. Numerical experiments on MNIST under sign flip, ALIE~\cite{baruch2019little}, and inner product manipulation~\cite{xie2020fall} attacks demonstrate that GT-PD-L achieves the highest accuracy across all attack strategies, outperforming coordinate-wise trimmed mean by up to $4.3$ percentage points under stealth attacks, exhibiting a defense-in-depth property whereby the projection and leaky integrator layers maintain convergence even when the dropout layer is defeated.


\section{Problem Setting and GT-PD}
\label{sec:formulation}

In this section, we formally describe the distributed optimization problem in the presence of Byzantine agents, introduce the probabilistic dropout mechanism, and present the GT-PD algorithm. We begin by specifying the network model and the objective.

\subsection{Network and Agents}
Consider a network of $n$ agents indexed by $V = \{1,\ldots,n\}$ communicating over an undirected base graph $G = (V,E)$. A subset $\cB \subset V$ with $|\cB| = b$ consists of Byzantine agents whose updates are arbitrary and adversarial. The honest agents $\cH = V \setminus \cB$ (with $n_\cH \triangleq |\cH|$) collectively solve:
\begin{equation}
\label{eq:problem}
\min_{x \in \R^d}\; f_\cH(x) \triangleq \frac{1}{n_\cH} \sum_{i \in \cH} f_i(x).
\end{equation}

We impose the following standard regularity conditions.

\begin{assumption}[Smoothness]
\label{as:smooth}
Each $f_i$, $i \in \cH$, is $L$-smooth: $\|\nabla f_i(x) - \nabla f_i(y)\| \leq L\|x-y\|$ for all $x,y \in \R^d$.
\end{assumption}

\begin{assumption}[Strong convexity]
\label{as:strong_cvx}
The average objective $f_\cH$ is $\mu$-strongly convex with $\mu > 0$. Define the condition number $\kappa = L/\mu$.
\end{assumption}

Assumption~\ref{as:strong_cvx} guarantees a unique minimizer $x^* = \arg\min f_\cH(x)$. We define the gradient heterogeneity at the optimum as $\zeta_f^2 \triangleq \sum_{i \in \cH} \|\nabla f_i(x^*)\|^2$, which quantifies the degree of non-IID data across agents; it is finite under Assumption~\ref{as:smooth} and vanishes if and only if all local objectives share the same minimizer. Together, Assumptions~\ref{as:smooth} and~\ref{as:strong_cvx} are the standard regularity requirements for proving linear convergence in gradient tracking algorithms \cite{nedic2017achieving,pu2021push,pu2021distributed}. Note that only the \emph{average} honest objective is required to be strongly convex; individual $f_i$ need not be convex.

\begin{remark}[Timeline, filtrations, and adversary model]
\label{rem:filtration}
One iteration of GT-PD proceeds in the following order: (a) the states $(x_i^k, y_i^k)_{i\in\cH}$ are in place; (b) messages for round $k$ are formed and exchanged (Byzantine agents send arbitrary vectors); (c) the dropout mask $\{\xi_{ij}^k\}$ is sampled; (d) the decision update produces $x_i^{k+1}$; (e) the iteration-$(k{+}1)$ oracle samples are drawn at $x_i^{k+1}$; (f) the tracker update produces $y_i^{k+1}$. Accordingly we work with two interleaved $\sigma$-algebras:
\begin{equation*}
\cG^k \triangleq \sigma\Big(\{\xi_i^l\}_{l\le k-1},\ \{\xi_{ij}^l\}_{l\le k-1},\ \Xi^{k-1}_{\mathrm{adv}}\Big),
\qquad
\cF^k \triangleq \cG^k \vee \sigma\big(\{\xi_i^k\}_{i\in\cH}\big),
\end{equation*}
where $\xi_i^l$ are the oracle samples, $\xi_{ij}^l$ the dropout masks, and $\Xi^{k-1}_{\mathrm{adv}}$ collects any randomness used by the adversary through round $k-1$. Thus $\cG^k$ is the information available \emph{before the iteration-$k$ oracle samples} are drawn, and $\cF^k$ is the information available \emph{before the iteration-$k$ dropout mask} is sampled. Under this convention: $x_i^k$ is $\cG^k$-measurable; $y_i^k$ is $\cF^k$-measurable (it contains $g_i^k$); and the retention probabilities $p_{ij}^k$ of Definition~\ref{def:data_driven_p}, which depend only on $x^{k-1}$ and $y^{k-1}$, are $\cF^{k-1}$-measurable, hence $\cF^k$-measurable. The Byzantine agents are \emph{omniscient but not prescient}: their round-$k$ messages may depend on all randomness realized through round $k$ (including the current samples $g_i^k$ through the exchanged trackers), but not on randomness not yet realized. Consequently, we never assume that adversarial quantities are uncorrelated with same-round stochastic noise; every cross-term involving an adversarial perturbation is handled by Young's or Cauchy--Schwarz inequalities, while zero-mean arguments are reserved for the honest oracle noise against $\cG^k$- or $\cF^k$-measurable quantities.
\end{remark}

\begin{assumption}[Stochastic oracle]
\label{as:stochastic}
At each iteration $k$, every honest agent $i \in \cH$ draws a sample $\xi_i^k$ and forms the stochastic gradient $g_i^k = \nabla F_i(x_i^k, \xi_i^k)$. Conditioned on $\cG^k$, the samples $\{\xi_i^k\}_{i \in \cH}$ are independent across agents, and
\begin{equation*}
\E[g_i^k \,|\, \cG^k] = \nabla f_i(x_i^k), \qquad \E\big[\|g_i^k - \nabla f_i(x_i^k)\|^2 \,|\, \cG^k\big] \leq \sigma^2 .
\end{equation*}
(Note that $x_i^k$ is $\cG^k$-measurable, so the conditioning is well posed.) Moreover, the iteration-$k$ samples are conditionally independent of the round-$(k{-}1)$ adversarial messages given $\cG^k$; this formalizes non-prescience of the adversary and is all that the analysis uses.
\end{assumption}

\subsection{Base Mixing Matrix}
\begin{definition}[Base mixing matrix]
\label{def:base_W}
Let $W \in \R^{n \times n}$ be a symmetric doubly stochastic matrix compatible with $G$: $W_{ij} > 0$ if $(i,j) \in E$ or $i = j$, $W_{ij} = 0$ otherwise, $W\ones = \ones$, and $\ones^\top W = \ones^\top$. In particular the diagonal of $W$ is strictly positive.
\end{definition}

A concrete construction is the Metropolis-Hastings rule: $W_{ij} = 1/(1 + \max\{d_i, d_j\})$ for $(i,j) \in E$, with $W_{ii} = 1 - \sum_{j \neq i} W_{ij}$, where $d_i$ is the degree of node $i$.

Double stochasticity of the mixing matrix is crucial for gradient tracking: it ensures that the consensus averaging preserves the mean of the agents' iterates, which in turn allows the gradient tracker to converge to the exact global gradient rather than a biased estimate. The dropout mechanism in Definition~\ref{def:dropout} is specifically designed to preserve this property.

\subsection{Probabilistic Edge Dropout}
\begin{definition}[Dropout mask]
\label{def:dropout}
At iteration $k$, each edge $(i,j) \in E$ is independently retained with probability $p_{ij}^k \in (0,1]$ and dropped with probability $1-p_{ij}^k$. Formally, let $\xi_{ij}^k \sim \text{Bernoulli}(p_{ij}^k)$ with $\xi_{ij}^k = \xi_{ji}^k$ (symmetric dropout, realized by a shared per-edge random seed). The dropout mixing matrix is:
\begin{equation}
\label{eq:Wk}
[W^k]_{ij} = \begin{cases} W_{ij} \cdot \xi_{ij}^k & i \neq j, \\ 1 - \sum_{j \neq i} W_{ij} \cdot \xi_{ij}^k & i = j. \end{cases}
\end{equation}
\end{definition}

\begin{proposition}[Doubly stochastic property]
\label{prop:ds}
If $W$ is symmetric, then $W^k$ defined by~\eqref{eq:Wk} is symmetric doubly stochastic almost surely for every realization of the dropout mask.
\end{proposition}
\begin{proof}
For row stochasticity: $\sum_{j} [W^k]_{ij} = [W^k]_{ii} + \sum_{j \neq i} W_{ij}\xi_{ij}^k = 1$ by definition of the diagonal. For symmetry: $[W^k]_{ij} = W_{ij}\xi_{ij}^k = W_{ji}\xi_{ji}^k = [W^k]_{ji}$ since $W$ is symmetric and $\xi_{ij}^k = \xi_{ji}^k$. Column stochasticity follows from symmetry and row stochasticity. Nonnegativity holds since $W_{ij} \geq 0$ and $\xi_{ij}^k \in \{0,1\}$; the diagonal satisfies $[W^k]_{ii} = 1 - \sum_{j \neq i} W_{ij}\xi_{ij}^k \geq 1 - \sum_{j \neq i} W_{ij} = W_{ii} > 0$.
\end{proof}

\begin{remark}
Proposition~\ref{prop:ds} highlights a structural advantage of the dropout mechanism. In spatial filtering defenses such as trimmed mean or Krum, each agent selectively excludes certain neighbors' updates based on their values. This exclusion removes weight from specific rows and columns of the mixing matrix without redistributing it, which breaks double stochasticity. Once double stochasticity is lost, the consensus average drifts: honest agents no longer converge to the true mean of their iterates but instead to a biased point that depends on the filtering rule \cite{wu2023byzantine}. GT-PD avoids this failure mode entirely. Because the diagonal of $W^k$ absorbs the weight of every dropped edge (see Definition~\ref{def:dropout}), the row and column sums remain equal to one for every realization of the dropout mask, regardless of how many edges are dropped or which retention probabilities are chosen.
\end{remark}

\begin{assumption}[Conditional independence of the dropout mask]
\label{as:independence}
At each iteration $k$, conditioned on $\cF^k$ (the information available before the iteration-$k$ mask is sampled, including the iteration-$k$ oracle samples; see Remark~\ref{rem:filtration}), the dropout variables $\{\xi_{ij}^k\}_{(i,j) \in E}$ are independent across edges with $\xi_{ij}^k \,|\, \cF^k \sim \text{Bernoulli}(p_{ij}^k)$, where $p_{ij}^k$ is $\cF^{k-1}$-measurable. There exists a constant $p_{\min} > 0$ such that $p_{ij}^k \geq p_{\min}$ for all \emph{honest-honest} edges $(i,j) \in \cE_\cH$ and all $k \geq 0$; retention probabilities on honest-Byzantine edges may take any values in $[0,1]$, including zero. Under the data-driven retention probability (Definition~\ref{def:data_driven_p}), the lower bound holds on every edge with $p_{\min} = e^{-\lambda(8-S_0)}$ (Proposition~\ref{prop:p_lower}); restricting the assumption to $\cE_\cH$ is what makes the complete-isolation benchmark ($p_b = 0$) of Section~\ref{sec:convergence_exact} consistent with it.
\end{assumption}

Assumption~\ref{as:independence} states that, conditioned on everything realized before the iteration-$k$ mask draw, the dropout decisions are fresh coin flips, uncorrelated across edges \emph{and} uncorrelated with the current oracle samples. The latter point is deliberate: the tracker $y_i^k$ contains the sample $g_i^k$, so conditioning only on the pre-sample information would not justify applying the second-moment contraction (Lemma~\ref{lem:second_moment}) to the tracking disagreement. The retention probabilities $p_{ij}^k$ may depend on past iterates (as in Definition~\ref{def:data_driven_p}), but the only randomness in $W^k$ beyond $\cF^k$ is the collection of independent Bernoulli coin flips.

\subsection{Self-Centered Projection}
We introduce a message-level defense that bounds the magnitude of adversarial perturbations without modifying the mixing matrix.

\begin{definition}[Self-centered projection]
\label{def:clip}
For a receiving agent $i \in \cH$ and a projection radius $\tau > 0$, the self-centered projection of an incoming message $\tilde{x}_j^k$ is defined as
\begin{equation}
\label{eq:clip}
\cP_{\tau,i}(\tilde{x}_j^k) \triangleq x_i^k + \min\!\left(1,\; \frac{\tau}{\|\tilde{x}_j^k - x_i^k\|}\right)(\tilde{x}_j^k - x_i^k).
\end{equation}
If $\|\tilde{x}_j^k - x_i^k\| \leq \tau$, the message passes through unmodified. Otherwise, it is projected onto the boundary of the $\ell_2$-ball of radius $\tau$ centered at $x_i^k$, preserving the direction of $\tilde{x}_j^k - x_i^k$. When $\tilde{x}_j^k = x_i^k$, the projection returns $x_i^k$ by convention. The tracker channel uses its own receiver-centered copy of the projection: incoming tracker messages $\tilde{y}_j^k$ are projected onto the $\tau$-ball centered at the receiver's \emph{tracker} $y_i^k$, i.e., $\cP_{\tau,i}^{y}(\tilde{y}_j^k) \triangleq y_i^k + \min\big(1, \tau/\|\tilde{y}_j^k - y_i^k\|\big)(\tilde{y}_j^k - y_i^k)$. Centering the tracker projection at $y_i^k$ (never at $x_i^k$) is essential: it is what makes the clipping corrections antisymmetric within each channel (Lemma~\ref{lem:honest_clip}) and what bounds the tracker clipping error by the tracking disagreement. We write $\cP_{\tau,i}$ for both channels when the center is clear from context.
\end{definition}

In the GT-PD framework, the self-centered projection is applied universally to all incoming messages, regardless of the sender's identity. No detection decision is required to determine which messages are projected.

\begin{proposition}[Enforced perturbation bound]
\label{prop:enforced_bound}
Under self-centered projection with radius $\tau > 0$, the aggregate Byzantine perturbation on any honest agent $i \in \cH$ satisfies, deterministically and for all $k \geq 0$:
\begin{equation}
\label{eq:enforced_zeta}
\left\| \sum_{m \in \cB} [W^k]_{im}\left(\cP_{\tau,i}(\tilde{x}_m^k) - x_i^k\right)\right\| \leq \delta_i^{\cB}\, \tau,
\end{equation}
where $\delta_i^{\cB} \triangleq \sum_{m \in \cB} [W^k]_{im}$ is the aggregate weight of Byzantine neighbors of agent $i$ at iteration $k$. The bound holds regardless of the messages $\tilde{x}_m^k$ sent by Byzantine agents.
\end{proposition}
\begin{proof}
By Definition~\ref{def:clip}, $\|\cP_{\tau,i}(\tilde{x}_m^k) - x_i^k\| \leq \tau$ for all $m \in \cB$. Applying the triangle inequality:
\begin{align*}
\left\| \sum_{m \in \cB} [W^k]_{im}\left(\cP_{\tau,i}(\tilde{x}_m^k) - x_i^k\right)\right\|
&\leq \sum_{m \in \cB} [W^k]_{im}\, \|\cP_{\tau,i}(\tilde{x}_m^k) - x_i^k\|
\leq \sum_{m \in \cB} [W^k]_{im}\, \tau = \delta_i^{\cB}\, \tau.
\end{align*}
Since $W^k$ has non-negative entries and the bound is deterministic by construction, the result holds for every realization of the dropout mask and for all $k \geq 0$.
\end{proof}

We define the maximum base weight of Byzantine neighbors across honest agents, and the effective perturbation bound, as
\begin{equation*}
\delta_{\max}^{\cB} \triangleq \max_{i \in \cH} \sum_{m \in \cB} W_{im},
\qquad\quad
\zeta \triangleq \delta_{\max}^{\cB} \cdot \tau ;
\end{equation*}
since $[W^k]_{im}\le W_{im}$, we have $\delta_i^{\cB} \leq \delta_{\max}^{\cB}$. The same bound applies to the gradient tracker channel. The perturbation bound is therefore a structural design parameter controlled by the projection radius $\tau$, rather than a condition on the adversary's behavior: the adversary is free to send arbitrary, unbounded messages, and the self-centered projection ensures that only a bounded perturbation reaches the consensus update.

The second structural property of the projection — exact preservation of the honest network average, whether or not it activates on honest messages — is established in Lemma~\ref{lem:honest_clip} below, once the effective honest subsystem is in place.

\begin{remark}[Separation of defense layers]
\label{rem:detection}
In the GT-PD framework, the two defense layers operate independently and serve complementary roles. The self-centered projection (Definition~\ref{def:clip}) is applied universally to all incoming messages and requires no neighbor classification; it bounds the magnitude of any message that deviates beyond $\tau$ from the receiving agent's state, and it never biases the honest average (Lemma~\ref{lem:honest_clip}). Under either isolation regime, honest messages fall within the $\tau$-ball with high probability over any fixed horizon (Propositions~\ref{prop:honest_nonclip} and~\ref{prop:honest_nonclip_partial}). The data-driven retention probability (Definition~\ref{def:data_driven_p}) governs the dropout mechanism and determines the frequency of communication with each neighbor. Agents whose parameters or gradient trackers deviate substantially from their neighbors in either channel receive low retention probabilities and are dropped with high probability. The detection literature provides alternative dropout score mechanisms, including credibility-assessment-based scoring \cite{hou2022credibility}, local-similarity filtering \cite{fang2024byzantine}, and dual-domain clustering \cite{sun2024dfl}, any of which can replace Definition~\ref{def:data_driven_p} without affecting the projection layer.
\end{remark}

\subsection{GT-PD Algorithm}
We now present the GT-PD algorithm. Each honest agent $i \in \cH$ maintains a decision variable $x_i^k \in \R^d$ and a gradient tracker $y_i^k \in \R^d$, updated as:
\begin{align}
x_i^{k+1} &= \sum_{j=1}^{n} [W^k]_{ij}\, \cP_{\tau,i}(\hat{x}_j^k) - \alpha\, y_i^k, \label{eq:x_update}\\
y_i^{k+1} &= \sum_{j=1}^{n} [W^k]_{ij}\, \cP_{\tau,i}(\hat{y}_j^k) + g_i^{k+1} - g_i^k, \label{eq:y_update}
\end{align}
where $\hat{x}_j^k = x_j^k$ for $j \in \cH$ and $\hat{x}_j^k = \tilde{x}_j^k$ for $j \in \cB$ (and similarly for $\hat{y}_j^k$), $\alpha > 0$ is the step size, $\cP_{\tau,i}$ is the self-centered projection from Definition~\ref{def:clip}, and $y_i^0 = g_i^0$ for all $i \in \cH$.

The rationale behind GT-PD is that each agent applies the self-centered projection $\cP_{\tau,i}$ to \emph{every} incoming message before incorporating it into the consensus step. For honest neighbors near consensus, the projection is inactive with high probability (Propositions~\ref{prop:honest_nonclip} and~\ref{prop:honest_nonclip_partial}). For Byzantine agents sending adversarial messages, the projection enforces a per-message magnitude bound of $\tau$ (Proposition~\ref{prop:enforced_bound}). The data-driven retention probability $p_{ij}^k$ (Definition~\ref{def:data_driven_p}) operates independently as the dropout mechanism, controlling the \emph{frequency} of adversarial exposure. The expected aggregate Byzantine perturbation per iteration is therefore bounded by $\delta_i^{\cB} \cdot p_b \cdot \tau$, which vanishes under complete isolation ($p_b = 0$) regardless of $\tau$. The stochastic gradient correction term $g_i^{k+1} - g_i^k$ maintains the gradient tracking property, ensuring that the agents' gradient estimators track the true global gradient of the honest objective $f_\cH$ in expectation.

\subsection{Effective Honest Subsystem}
To precisely analyze the convergence of GT-PD without being contaminated by unbounded adversarial states, we isolate the dynamics over the honest agents.

\begin{definition}[Effective honest mixing matrix]
\label{def:effective_W}
For the honest agents $\cH$, we define the effective mixing matrix $\hat{W}^k \in \R^{|\cH| \times |\cH|}$ by redistributing each honest agent's weight on Byzantine neighbors back to its self-loop. Formally, for $i,j \in \cH$:
\begin{equation}
[\hat{W}^k]_{ij} = \begin{cases} [W^k]_{ij} & i \neq j, \\ [W^k]_{ii} + \sum_{m \in \cB} [W^k]_{im} & i = j. \end{cases}
\end{equation}
Because $W^k$ is symmetric and doubly stochastic over all $n$ agents (Proposition~\ref{prop:ds}), it follows that $\hat{W}^k$ is symmetric and doubly stochastic over the $n_\cH$ honest nodes almost surely. We also define the honest consensus projection matrix $J_\cH = \frac{1}{n_\cH}\ones\ones^\top \in \R^{n_\cH \times n_\cH}$; the matrix $(I_\cH - J_\cH)$ projects onto the disagreement subspace.
\end{definition}

We define the three quantities that form the state vector of the convergence analysis:
\begin{equation}
\label{eq:lyap_quantities}
\begin{split}
u^k &\triangleq \xbar^k - x^*, \qquad
C^k \triangleq \Fnorm{(I_\cH - J_\cH)\mathbf{X}_\cH^k}^2
= \sum_{i \in \cH}\|x_i^k - \xbar^k\|^2, \\
T^k &\triangleq \Fnorm{(I_\cH - J_\cH)\mathbf{Y}_\cH^k}^2
= \sum_{i \in \cH}\|y_i^k - \ybar^k\|^2,
\end{split}
\end{equation}
where $\xbar^k = \frac{1}{n_\cH}\sum_{i \in \cH} x_i^k$ and $\ybar^k = \frac{1}{n_\cH}\sum_{i \in \cH} y_i^k$ are the honest averages. The quantity $u^k$ measures the optimality gap, $C^k$ the consensus disagreement among honest agents, and $T^k$ the tracking disagreement of the gradient estimators.

In compact matrix notation, letting $\mathbf{X}_\cH^k \in \R^{n_\cH \times d}$ stack the honest agents' states, we rewrite the updates strictly over the $\cH$-subspace:
\begin{align}
\mathbf{X}_\cH^{k+1} &= \hat{W}^k \mathbf{X}_\cH^k - \alpha \mathbf{Y}_\cH^k + \mathbf{E}_{x,\textup{clip}}^k + \mathbf{E}_{x,\cB}^k, \label{eq:X_update_H}\\
\mathbf{Y}_\cH^{k+1} &= \hat{W}^k \mathbf{Y}_\cH^k + \mathbf{G}_\cH^{k+1} - \mathbf{G}_\cH^k + \mathbf{E}_{y,\textup{clip}}^k + \mathbf{E}_{y,\cB}^k, \label{eq:Y_update_H}
\end{align}
where the $i$-th row of the Byzantine perturbation matrix $\mathbf{E}_{x,\cB}^k$ is $\sum_{m \in \cB} [W^k]_{im}(\cP_{\tau,i}(\tilde{x}_m^k) - x_i^k)$, the honest clipping matrices $\mathbf{E}_{x,\textup{clip}}^k, \mathbf{E}_{y,\textup{clip}}^k$ are those of Lemma~\ref{lem:honest_clip} (identically zero whenever no honest message is clipped), and the $i$-th row of $\mathbf{G}_\cH^k$ stacks the local stochastic gradients $g_i^k$. By Proposition~\ref{prop:enforced_bound} and $\zeta = \delta_{\max}^{\cB}\, \tau$, every row of $\mathbf{E}_{x,\cB}^k$ and $\mathbf{E}_{y,\cB}^k$ has norm at most $\zeta$, so $\Fnorm{\mathbf{E}_{x,\cB}^k} \leq \sqrt{n_\cH}\,\zeta$ and $\Fnorm{\mathbf{E}_{y,\cB}^k} \leq \sqrt{n_\cH}\,\zeta$, and the honest averages satisfy $\|\bar e^k_{x,\cB}\|\le\zeta$, $\|\bar e^k_{y,\cB}\|\le\zeta$.

\begin{lemma}[Honest clipping: exact average preservation and disagreement bound]
\label{lem:honest_clip}
Under universal self-centered projection with radius $\tau>0$, define the honest clipping error matrices with $i$-th rows
$[\mathbf{E}_{x,\textup{clip}}^k]_i = \sum_{j \in \cH} [\hat{W}^k]_{ij}\big(\cP_{\tau,i}(x_j^k) - x_j^k\big)$
and
$[\mathbf{E}_{y,\textup{clip}}^k]_i = \sum_{j \in \cH} [\hat{W}^k]_{ij}\big(\cP_{\tau,i}(y_j^k) - y_j^k\big)$,
where $\hat W^k$ is the effective honest matrix of Definition~\ref{def:effective_W}. Then, for every realization:
\begin{equation}
\Fnorm{\mathbf{E}_{x,\textup{clip}}^k}^2 \leq 4C^k, \qquad
\Fnorm{\mathbf{E}_{y,\textup{clip}}^k}^2 \leq 4T^k,
\end{equation}
and the honest averages are preserved \emph{exactly}:
\begin{equation}
\label{eq:clip_cancel}
\bar e_{x,\textup{clip}}^k \triangleq \frac{1}{n_\cH}\sum_{i\in\cH}[\mathbf{E}_{x,\textup{clip}}^k]_i = 0,
\qquad
\bar e_{y,\textup{clip}}^k = 0 .
\end{equation}
\end{lemma}
\begin{proof}
Write $\cP_{\tau,i}(x_j^k) - x_j^k = (c_{ij}^k - 1)(x_j^k - x_i^k)$ with $c_{ij}^k = \min\{1, \tau/\|x_j^k - x_i^k\|\}\in(0,1]$ (and $c^k_{ii}=1$). Since $|c_{ij}^k - 1| \leq 1$, Jensen's inequality with the row-stochastic weights $[\hat W^k]_{ij}$ gives $\|[\mathbf{E}_{x,\textup{clip}}^k]_i\|^2 \le \sum_{j\in\cH}[\hat W^k]_{ij}\|x_j^k-x_i^k\|^2$; summing over $i$ and applying the argument of Lemma~\ref{lem:deterministic_mixing} yields $\Fnorm{\mathbf{E}_{x,\textup{clip}}^k}^2 \le 4C^k$. For the averages, note that $c_{ij}^k = c_{ji}^k$ (it depends only on $\|x_j^k - x_i^k\|$ and the common radius $\tau$) and $[\hat{W}^k]_{ij} = [\hat{W}^k]_{ji}$ (Proposition~\ref{prop:ds} and Definition~\ref{def:effective_W}). Hence, in
$\sum_{i,j \in \cH} [\hat{W}^k]_{ij}\,(c_{ij}^k - 1)(x_j^k - x_i^k)$,
the $(i,j)$ and $(j,i)$ terms are exact negatives of each other, and the sum vanishes for every realization. The $y$-channel argument is identical.
\end{proof}

\begin{remark}[Why average preservation matters]
\label{rem:clip_ds}
The self-centered projection modifies only message values; the mixing matrix $W^k$ is never altered, so the doubly stochastic property of Proposition~\ref{prop:ds} is preserved. Lemma~\ref{lem:honest_clip} adds a second exactness property: even when the projection \emph{does} activate on honest messages, the induced corrections cancel pairwise in the network average. Consequently the recursions for the average iterate and for the tracking-invariant violation in Sections~\ref{sec:convergence_exact} and~\ref{sec:convergence_leaky} are exact for every realization, and honest clipping can affect only the \emph{disagreement} dynamics ($C^k$ and $T^k$). This single observation removes any feedback loop between honest clipping and the optimality gap, and is the structural reason the convergence guarantees below carry no clipping-induced coupling term. An equivalent way to phrase it: the honest part of the clipped consensus step equals mixing by the matrix $\tilde W^k$ with $[\tilde W^k]_{ij} = [\hat W^k]_{ij} c^k_{ij}$ for $i\neq j$ and the complementary diagonal, which is again symmetric and doubly stochastic; clipping reweights the mixing but never breaks its structure.
\end{remark}

\begin{remark}
\label{rem:tracking_invariant}
Equations~\eqref{eq:X_update_H}--\eqref{eq:Y_update_H} decompose the GT-PD dynamics into (a) the honest consensus-plus-gradient-tracking dynamics governed by the doubly stochastic matrix $\hat{W}^k$, and (b) additive perturbation terms capturing honest clipping and the bounded residual Byzantine influence. A critical observation is that the Byzantine perturbation $\mathbf{E}_{y,\cB}^k$ in the $y$-update causes the standard gradient tracking invariant to be violated. Defining $s^k = \ybar^k - \frac{1}{n_\cH}\sum_i g_i^k$ as the tracking-invariant violation, averaging~\eqref{eq:Y_update_H} and using Lemma~\ref{lem:honest_clip} gives $s^{k+1} = s^k + \bar{e}_{y,\cB}^k$: pure accumulation with no contraction. (Note that honest clipping contributes nothing to this recursion, by~\eqref{eq:clip_cancel}.) This prevents the standard analysis from yielding convergence to a fixed neighborhood when $\zeta > 0$ and $p_b>0$. We address this through two complementary results: an exact guarantee under complete isolation (Section~\ref{sec:convergence_exact}), and the leaky-tracking variant GT-PD-L under partial isolation (Section~\ref{sec:convergence_leaky}).
\end{remark}

\begin{assumption}[Expected connectivity of the honest subgraph]
\label{as:connectivity}
The expected mixing matrix restricted to honest agents, $\bar{W}_\cH \triangleq \E[\hat{W}^k]$, is irreducible and aperiodic, with spectral gap $1 - \lambda_2(\bar{W}_\cH) > 0$. Equivalently (since $p_{ij}^k\ge p_{\min}>0$, the support of $\bar W_\cH$ coincides with that of the all-edges-retained effective matrix $\hat W$), the honest-induced base subgraph $G|_\cH$ is connected.
\end{assumption}

\subsection{Data-Driven Retention Probability}
\label{sec:data_driven}
The self-centered projection provides magnitude control through a purely geometric mechanism that requires no neighbor classification. The second defense layer, probabilistic edge dropout, requires retention probabilities $\{p_{ij}^k\}$ that encode the trust each agent places in its neighbors. We now present a concrete, fully decentralized estimator computed locally from the spatial parameters and gradient tracker variables exchanged during the consensus step. The estimator requires no additional communication beyond what GT-PD already performs.

\begin{definition}[Threshold-gated data-driven retention probability]
\label{def:data_driven_p}
Fix a tolerance $S_0 \in (0,8)$ and a sensitivity $\lambda > 0$. At each iteration $k \geq 1$, each pair of communicating agents $(i,j) \in E$ computes
\begin{equation}
\label{eq:data_driven_p}
p_{ij}^k \;=\;
\begin{cases}
1, & S_{ij}^k \le S_0,\\[2pt]
\exp\!\big(-\lambda\,(S_{ij}^k - S_0)\big), & S_{ij}^k > S_0,
\end{cases}
\end{equation}
where the composite trust score $S_{ij}^k = S_{ij}^{k,x} + S_{ij}^{k,y}$ aggregates normalized deviations in the spatial and temporal channels:
\begin{align}
S_{ij}^{k,x} &= \frac{\|x_j^{k-1} - x_i^{k-1}\|^2}
  {\frac{1}{2}\!\left(\|x_i^{k-1}\|^2
  + \|x_j^{k-1}\|^2\right) + \eta_x^2},
  \label{eq:score_x} \\
S_{ij}^{k,y} &= \frac{\|y_j^{k-1} - y_i^{k-1}\|^2}
  {\frac{1}{2}\!\left(\|y_i^{k-1}\|^2
  + \|y_j^{k-1}\|^2\right) + \eta_y^2}.
  \label{eq:score_y}
\end{align}
Here $\eta_x, \eta_y > 0$ are regularization parameters preventing degeneracy when the respective variables approach zero near convergence. For the initial iteration $k = 0$, all edges are retained with probability $p_{ij}^0 = 1$. For Byzantine senders the scores are computed from their reported messages; on honest-honest edges the scores, hence the retention probabilities and the shared-seed mask draws, are symmetric.
\end{definition}

The threshold gate protects communication near consensus: honest edges whose composite score falls below the tolerance $S_0$ are retained \emph{with probability one}, deterministically, which is what makes the zero-penalty regime of Lemma~\ref{lem:two_tier_spectral} reachable in finite time rather than only asymptotically (see Corollary~\ref{cor:S0_rule}). Above the threshold, retention decays exponentially in the excess score.

\begin{proposition}[Uniform lower bound on retention probability]
\label{prop:p_lower}
Under Definition~\ref{def:data_driven_p}, for all $(i,j) \in E$ and all $k \geq 0$:
\begin{equation}
\label{eq:p_lower}
p_{ij}^k \geq e^{-\lambda(8-S_0)} \triangleq p_{\min} > 0.
\end{equation}
\end{proposition}
\begin{proof}
We bound each channel score. Let $a = x_i^{k-1}$, $b = x_j^{k-1}$. By the parallelogram law, $\|a - b\|^2 \leq 2\|a\|^2 + 2\|b\|^2$, while the denominator of~\eqref{eq:score_x} is at least $\frac{1}{2}(\|a\|^2 + \|b\|^2)$ when $(a,b)\neq(0,0)$, so $S_{ij}^{k,x} \le 4$; when $a=b=0$ the numerator vanishes and the regularizer keeps the denominator positive, so $S_{ij}^{k,x}=0$. An identical argument yields $S_{ij}^{k,y} \leq 4$. Therefore $S_{ij}^k \le 8$, and~\eqref{eq:data_driven_p} gives $p_{ij}^k \ge \exp(-\lambda(8-S_0))$; at $k=0$, $p^0_{ij}=1$.
\end{proof}

\begin{remark}[Multiplicative factorization and channel independence]
\label{rem:dual_metric}
Above the gate, the additive composite score yields the multiplicative factorization
\begin{equation}
\label{eq:p_factorization}
e^{-\lambda(S^k_{ij}-S_0)} \;=\; e^{\lambda S_0}\cdot \underbrace{e^{-\lambda S_{ij}^{k,x}}}_{p_{ij}^{k,x}} \cdot \underbrace{e^{-\lambda S_{ij}^{k,y}}}_{p_{ij}^{k,y}},
\end{equation}
so a large deviation in \emph{either} the parameter space or the gradient tracker space drives the retention probability toward zero, even if the other channel exhibits no deviation. This prevents a Byzantine agent from evading detection by crafting adversarial messages that are anomalous in only one channel (Proposition~\ref{prop:decoupled_failure}). The scores are symmetric ($S_{ij}^k = S_{ji}^k$, hence $p_{ij}^k = p_{ji}^k$) and $\cF^{k-1}$-measurable, since $x^{k-1}$ and $y^{k-1}$ are determined before round $k$. The formulation generalizes to channel-specific sensitivities $\lambda_x,\lambda_y$ with only notational changes in Proposition~\ref{prop:p_lower}.
\end{remark}


\section{Properties of the Dropout Mixing Matrix}
\label{sec:dropout_properties}

In this section, we establish the spectral properties of the effective honest mixing matrix $\hat{W}^k$ that dictate the consensus dynamics in GT-PD. Throughout, $\hat W$ denotes the effective honest matrix obtained from the base matrix $W$ with all edges retained, and
\begin{equation}
\label{eq:lambdahat}
\hat\lambda^2 \;\triangleq\; \big\|\hat W - J_\cH\big\|_2^2 \;=\; \max\big\{\lambda_2(\hat W)^2,\ \lambda_{n_\cH}(\hat W)^2\big\}.
\end{equation}
We work with the spectral norm rather than $\lambda_2$ alone because the smallest eigenvalue of a doubly stochastic matrix also controls the contraction; the two coincide only when $|\lambda_{n_\cH}(\hat W)| \le \lambda_2(\hat W)$.

\begin{lemma}[Spectral gap of the all-retained matrix]
\label{lem:base_gap}
Under Definition~\ref{def:base_W} and Assumption~\ref{as:connectivity}, $\hat\lambda < 1$.
\end{lemma}
\begin{proof}
$\hat W$ is symmetric doubly stochastic and, by Assumption~\ref{as:connectivity}, irreducible; by Perron--Frobenius its eigenvalue $1$ is simple with eigenspace $\mathrm{span}(\ones)$, so $\lambda_2(\hat W) < 1$. For the lower spectral edge, Definition~\ref{def:base_W} gives $W_{ii} > 0$, and Definition~\ref{def:effective_W} gives $[\hat W]_{ii} \ge W_{ii}$; by the Gershgorin circle theorem and row stochasticity, every eigenvalue $\lambda$ of $\hat W$ satisfies $\lambda \ge 2\min_i [\hat W]_{ii} - 1 \ge 2\min_i W_{ii} - 1 > -1$. Hence $\hat\lambda = \max\{\lambda_2, |\lambda_{n_\cH}|\} < 1$.
\end{proof}

\begin{lemma}[Expected mixing matrix under uniform dropout]
\label{lem:expected_W}
Suppose $p_{ij} = p$ for all $(i,j) \in E$ (uniform dropout). Then
\begin{equation}
\E[\hat{W}^k] = p\,\hat{W} + (1-p)\,I_\cH,
\end{equation}
and the spectral gap is $1 - \lambda_2(\E[\hat{W}^k]) = p\,(1 - \lambda_2(\hat{W}))$.
\end{lemma}
\begin{proof}
For $i \neq j \in \cH$: $\E[[\hat{W}^k]_{ij}] = W_{ij} \cdot \E[\xi_{ij}^k] = p\,W_{ij} = p[\hat{W}]_{ij}$. For the diagonal: $\E[[\hat{W}^k]_{ii}] = 1 - \sum_{j \in \cH, j \neq i} p\,W_{ij} = 1 - p(1 - [\hat{W}]_{ii}) = (1-p) + p\,[\hat{W}]_{ii}$ (the Byzantine mask variables cancel in the diagonal; see the display in the proof of Lemma~\ref{lem:two_tier_spectral}). Hence $\E[\hat{W}^k] = p\hat{W} + (1-p)I_\cH$; the eigenvalues are $p\lambda_i(\hat{W}) + (1-p)$, which gives the stated gap.
\end{proof}

Lemma~\ref{lem:expected_W} analyzes the uniform case as a structural baseline. Under Definition~\ref{def:data_driven_p} the dropout is non-uniform, and the relevant object for the squared-norm analysis is the following second-moment factor.

\begin{definition}[Second-moment contraction factor]
\label{def:sigma_p}
For $k\ge 0$ define, with the conditioning of Assumption~\ref{as:independence},
\begin{equation}
\sigma_p^{2,k} \triangleq \lambda_{\max}\!\left(\E\!\left[(\hat{W}^k - J_\cH)^\top (\hat{W}^k - J_\cH) \,\big|\, \cF^k\right]\right),
\qquad
\bar{\sigma}_p^2 \triangleq \sup_{k \geq 0}\; \operatorname{ess\,sup}\ \sigma_p^{2,k}.
\label{eq:sigma_bar}
\end{equation}
Throughout the convergence analysis we also use the shorthand
\begin{equation}
\label{eq:theta_def}
\thp \;\triangleq\; \frac{1+\sbp}{1-\sbp}\ \le\ \frac{2}{1-\sbp}.
\end{equation}
\end{definition}

The next lemma replaces a qualitative claim ($\sigma_p^{2,k}<1$ pointwise) by a quantitative bound that is uniform over iterations \emph{and} over all admissible histories; uniformity is exactly what the supremum in~\eqref{eq:sigma_bar} requires.

\begin{lemma}[Quantitative uniform contraction]
\label{lem:sigma_bar_bound}
Let $|\cE_\cH|$ denote the number of honest-honest edges of $G$. Under Assumptions~\ref{as:independence} and~\ref{as:connectivity},
\begin{equation}
\label{eq:sigma_quant}
\bar{\sigma}_p^2 \;\le\; 1 - p_{\min}^{|\cE_\cH|}\,\big(1 - \hat\lambda^2\big) \;<\; 1 .
\end{equation}
\end{lemma}
\begin{proof}
Fix $k$ and condition on $\cF^k$. Let $\cE_k$ denote the event that every honest-honest edge is retained at iteration $k$. By Assumption~\ref{as:independence} (conditional independence across edges and $p_{ij}^k \ge p_{\min}$),
$\Prob(\cE_k \,|\, \cF^k) \ge p_{\min}^{|\cE_\cH|}$.
On $\cE_k$, the Byzantine mask variables cancel in the diagonal of $\hat W^k$ (the display in the proof of Lemma~\ref{lem:two_tier_spectral}), so the honest block equals the deterministic matrix $\hat W$, and for any unit $v \perp \ones$,
$\|(\hat{W}^k - J_\cH)v\|^2 \le \hat\lambda^2$
by~\eqref{eq:lambdahat}. Off $\cE_k$, double stochasticity (Proposition~\ref{prop:ds}) gives $\|(\hat{W}^k - J_\cH)v\| \le \|v\| = 1$. Hence, for every unit $v\perp\ones$,
\begin{equation*}
\E\big[\|(\hat{W}^k - J_\cH)v\|^2 \,\big|\, \cF^k\big]
\;\le\; \Prob(\cE_k|\cF^k)\,\hat\lambda^2 + \big(1-\Prob(\cE_k|\cF^k)\big)
\;\le\; 1 - p_{\min}^{|\cE_\cH|}\big(1-\hat\lambda^2\big),
\end{equation*}
uniformly over $v$, over $k$, and over the realization of $\cF^k$. Taking the maximum over unit $v\perp\ones$ (the quadratic form of the conditional second-moment matrix on $\ones^\perp$) and then the supremum in~\eqref{eq:sigma_bar} yields~\eqref{eq:sigma_quant}; $\hat\lambda<1$ is Lemma~\ref{lem:base_gap}.
\end{proof}

The bound~\eqref{eq:sigma_quant} is exponentially conservative in $|\cE_\cH|$; it is a worst-case certificate, not a performance estimate. Under the threshold-gated rule, honest-honest edges satisfy $p^k_{ij}=1$ whenever their scores fall below $S_0$ (Corollary~\ref{cor:S0_rule}), in which case the event $\cE_k$ has probability one over the honest block and the effective contraction factor improves to $\hat\lambda^2$ (Lemma~\ref{lem:two_tier_spectral}).

\begin{lemma}[Second-moment contraction]
\label{lem:second_moment}
Under Assumption~\ref{as:independence}:
\begin{enumerate}[label=(\alph*)]
\item $\sigma_p^{2,k} \geq \big\|\E[\hat{W}^k \,|\, \cF^k] - J_\cH\big\|_2^2$.
\item For any $\cF^k$-measurable matrix $\mathbf{V} \in \R^{n_\cH \times d}$ satisfying $J_\cH \mathbf{V} = 0$:
\begin{equation}
\label{eq:second_moment_bound}
\E\!\left[\Fnorm{(\hat{W}^k - J_\cH)\mathbf{V}}^2 \,\big|\, \cF^k\right] \leq \bar{\sigma}_p^2\, \Fnorm{\mathbf{V}}^2.
\end{equation}
\end{enumerate}
\end{lemma}
\begin{proof}
\textbf{(a)} Let $M \triangleq \E[\hat W^k|\cF^k]$. The conditional variance decomposition
$\E[(\hat W^k - J_\cH)^\top(\hat W^k - J_\cH)|\cF^k] - (M-J_\cH)^\top(M-J_\cH) = \E[(\hat W^k - M)^\top(\hat W^k - M)|\cF^k] \succeq 0$
gives $(M-J_\cH)^\top(M-J_\cH) \preceq \E[(\hat W^k - J_\cH)^\top(\hat W^k - J_\cH)|\cF^k]$, and taking $\lambda_{\max}$ yields the claim.

\textbf{(b)} Since $J_\cH \mathbf{V} = 0$, every column of $\mathbf V$ lies in $\ones^\perp$. Writing $\Fnorm{(\hat{W}^k - J_\cH)\mathbf{V}}^2 = \tr(\mathbf{V}^\top (\hat{W}^k - J_\cH)^\top(\hat{W}^k - J_\cH)\mathbf{V})$ and taking conditional expectation (legitimate because $\mathbf{V}$ is $\cF^k$-measurable and, by Assumption~\ref{as:independence}, the mask is the only randomness beyond $\cF^k$):
\begin{align}
\E\big[\Fnorm{(\hat{W}^k - J_\cH)\mathbf{V}}^2\big|\cF^k\big] &= \tr\big(\mathbf{V}^\top \E\big[(\hat{W}^k - J_\cH)^\top(\hat{W}^k - J_\cH)\big|\cF^k\big]\mathbf{V}\big) \nonumber \\
&\leq \sigma_p^{2,k} \tr(\mathbf{V}^\top \mathbf{V}) \le \bar{\sigma}_p^2 \Fnorm{\mathbf{V}}^2,
\end{align}
where the first inequality restricts the quadratic form to $\ones^\perp$, which contains the columns of $\mathbf V$ (note $(\hat W^k - J_\cH)\ones = 0$, so the conditional second-moment matrix annihilates $\ones$).
\end{proof}

We also record a deterministic bound used repeatedly in the proofs.

\begin{lemma}[Deterministic consensus mixing bound]
\label{lem:deterministic_mixing}
For any doubly stochastic $\hat{W}^k$ and any $\mathbf{X} \in \R^{n_\cH \times d}$, let $\mathbf{X}_\perp = (I_\cH - J_\cH)\mathbf{X}$. Then
\begin{equation}
\sum_{i \in \cH} \left\|\sum_{j \in \cH} [\hat{W}^k]_{ij}(x_j - x_i)\right\|^2 \leq 4\Fnorm{\mathbf{X}_\perp}^2.
\end{equation}
\end{lemma}
\begin{proof}
By row stochasticity, $\sum_j [\hat{W}^k]_{ij}(x_j - x_i) = \sum_j [\hat{W}^k]_{ij}x_j - x_i$. Applying Jensen's inequality to $\|\cdot\|^2$:
$\|\sum_j [\hat{W}^k]_{ij}(x_j - x_i)\|^2 \leq \sum_j [\hat{W}^k]_{ij}\|x_j - x_i\|^2$.
Summing over $i$ and using $\|x_j - x_i\|^2 \leq 2\|x_j - \xbar\|^2 + 2\|x_i - \xbar\|^2$:
\begin{align}
\sum_i \sum_j [\hat{W}^k]_{ij}\|x_j - x_i\|^2 &\leq 2\sum_j\Big(\sum_i [\hat{W}^k]_{ij}\Big)\|x_j - \xbar\|^2 + 2\sum_i\Big(\sum_j [\hat{W}^k]_{ij}\Big)\|x_i - \xbar\|^2 \nonumber\\
&= 4\sum_i \|x_i - \xbar\|^2 = 4\Fnorm{\mathbf{X}_\perp}^2,
\end{align}
by column and row stochasticity.
\end{proof}

\begin{lemma}[Zero-Penalty Stability under Two-Tier Dropout]
\label{lem:two_tier_spectral}
Assume a two-tier dropout scheme where $p_{ij} = p_h$ for honest-honest edges and $p_{ij} = p_b$ for honest-Byzantine edges. If $p_h = 1$, the effective honest mixing matrix $\hat{W}^k$ is completely deterministic, and $\bar{\sigma}_p^2 = \|\hat W - J_\cH\|_2^2 = \hat\lambda^2$ for all $p_b \in [0, 1]$.
\end{lemma}
\begin{proof}
By Definition~\ref{def:effective_W}, the diagonal is $[\hat{W}^k]_{ii} = [W^k]_{ii} + \sum_{m \in \cB} [W^k]_{im}$. Expanding using $[W^k]_{ii} = 1 - \sum_{j \in \cH, j \neq i} W_{ij}\xi_{ij}^k - \sum_{m \in \cB} W_{im}\xi_{im}^k$:
\begin{align}
[\hat{W}^k]_{ii} &= 1 - \sum_{j \in \cH, j \neq i} W_{ij}\xi_{ij}^k - \sum_{m \in \cB} W_{im}\xi_{im}^k + \sum_{m \in \cB} W_{im}\xi_{im}^k = 1 - \sum_{j \in \cH, j \neq i} W_{ij}\xi_{ij}^k.
\label{eq:byz_cancel}
\end{align}
The Byzantine dropout terms cancel algebraically, so the honest block of $\hat W^k$ depends only on the honest-honest masks. If $p_h = 1$, then $\xi_{ij}^k = 1$ deterministically for all honest pairs, every entry of $\hat{W}^k$ equals the corresponding entry of $\hat W$, and Definition~\ref{def:sigma_p} reduces to $\bar\sigma_p^2 = \lambda_{\max}((\hat W - J_\cH)^\top(\hat W - J_\cH)) = \|\hat W - J_\cH\|_2^2$, independently of $p_b$.
\end{proof}

\begin{remark}
\label{rem:zero_penalty_practical}
Lemma~\ref{lem:two_tier_spectral} is the spectral foundation of the zero-penalty isolation guarantee: for any $p_b \in [0,1]$, the consensus dynamics over the honest subspace experience no variance penalty when $p_h = 1$, because the Byzantine mask variables cancel in~\eqref{eq:byz_cancel}. Under the data-driven rule, whether honest-honest edges reach $p^k_{ij} = 1$ is governed by the threshold gate. We emphasize a point that the stochastic analysis makes precise: with stochastic gradients, the tracking disagreement $T^k$ does \emph{not} vanish — it has a steady-state floor of order $n_\cH\sigma^2/(1-\sbp)^2$ (Corollary~\ref{cor:exact_steady}) — so the honest scores $S^{k,y}_{ij}$ remain bounded away from zero and, \emph{without} the gate ($S_0 = 0$), honest retention would remain bounded away from one for all time. The gate is therefore necessary, not cosmetic: once the disagreement floors are below the tolerance (Corollary~\ref{cor:S0_rule} quantifies when), honest edges are retained with probability one and the zero-penalty regime of this lemma is active. For the finite-time analysis before that point, the uniform bound $\bar{\sigma}_p^2$ of Lemma~\ref{lem:sigma_bar_bound} applies.
\end{remark}


\section{Convergence under Complete Isolation (\texorpdfstring{$p_b = 0$}{p\_b = 0})}
\label{sec:convergence_exact}

We first analyze the case where all edges to Byzantine agents are dropped deterministically, which we denote $p_b = 0$. In this regime, $[W^k]_{im} = 0$ for all $m \in \cB$, and consequently $\mathbf{E}_{x,\cB}^k = \mathbf{E}_{y,\cB}^k = 0$: the honest subsystem reduces to gradient tracking over a random doubly stochastic graph. We stress that $p_b = 0$ is an idealized \emph{analysis benchmark}, not an operational claim: Assumption~\ref{as:independence} requires the uniform retention lower bound only on honest-honest edges, so setting retention to zero on honest-Byzantine edges is consistent with it, whereas the data-driven rule of Definition~\ref{def:data_driven_p} keeps every retention probability strictly positive at finite $\lambda$ and therefore only approaches this benchmark as adversarial scores grow. The partial-isolation analysis of Section~\ref{sec:convergence_leaky} covers the realizable regime.

\subsection{Clean Dynamics under \texorpdfstring{$p_b = 0$}{p\_b = 0}}

Throughout this section we analyze the \emph{clean dynamics}, in which honest messages are not modified by the projection:
\begin{align}
\mathbf{X}_\cH^{k+1} &= \hat{W}^k \mathbf{X}_\cH^k - \alpha \mathbf{Y}_\cH^k, \label{eq:X_clean}\\
\mathbf{Y}_\cH^{k+1} &= \hat{W}^k \mathbf{Y}_\cH^k + \mathbf{G}_\cH^{k+1} - \mathbf{G}_\cH^k. \label{eq:Y_clean}
\end{align}
Proposition~\ref{prop:honest_nonclip} below closes the loop: it shows that, for a projection radius computable from the problem data, the actual GT-PD trajectory coincides with~\eqref{eq:X_clean}--\eqref{eq:Y_clean} over any fixed horizon with high probability, so the guarantees of this section apply verbatim to the projected algorithm.

\begin{lemma}[Tracking invariant under $p_b = 0$]
\label{lem:invariant}
With initialization $y_i^0 = g_i^0$ for all $i \in \cH$ and $p_b = 0$, the average gradient tracker satisfies
$\ybar^k = \frac{1}{n_\cH}\sum_{i \in \cH} g_i^k$ for all $k \geq 0$.
\end{lemma}
\begin{proof}
Left-multiplying~\eqref{eq:Y_clean} by $\frac{1}{n_\cH}\ones^\top$ and using $\ones^\top \hat{W}^k = \ones^\top$ yields $\ybar^{k+1} = \ybar^k + \frac{1}{n_\cH}\sum_{i}[g_i^{k+1} - g_i^k]$. Induction from $\ybar^0 = \frac{1}{n_\cH}\sum_i g_i^0$ completes the proof.
\end{proof}

We define the spatial tracking residual $\delta^k \triangleq \frac{1}{n_\cH}\sum_{i \in \cH}[\nabla f_i(x_i^k) - \nabla f_i(\xbar^k)]$ and the oracle noise $\bar\epsilon^k \triangleq \frac{1}{n_\cH}\sum_i (g_i^k - \nabla f_i(x_i^k))$. By Jensen's inequality and Assumption~\ref{as:smooth},
\begin{equation}
\label{eq:delta_bound}
\|\delta^k\|^2 \le \frac{1}{n_\cH}\sum_{i\in\cH}\|\nabla f_i(x_i^k)-\nabla f_i(\xbar^k)\|^2 \le \frac{L^2}{n_\cH} C^k,
\end{equation}
and by Assumption~\ref{as:stochastic} (conditional independence across agents given $\cG^k$), $\E[\bar\epsilon^k \,|\, \cG^k] = 0$ and $\E\|\bar\epsilon^k\|^2 \le \sigma^2/n_\cH$.

\subsection{One-Step Bounds}

\begin{lemma}[Optimality contraction under $p_b = 0$]
\label{lem:opt_exact}
Under Assumptions~\ref{as:smooth},~\ref{as:strong_cvx}, and~\ref{as:stochastic}, for a step size $\alpha \leq \frac{1}{2L}$:
\begin{equation}
\label{eq:U_bound}
\E[\|u^{k+1}\|^2] \leq \left(1 - \frac{\alpha\mu}{2}\right)\E[\|u^k\|^2] + \frac{5\alpha L^2}{\mu n_\cH}\,\E[C^k] + \frac{\alpha^2\sigma^2}{n_\cH}.
\end{equation}
\end{lemma}

\begin{proof}
Left-multiplying~\eqref{eq:X_clean} by $\frac{1}{n_\cH}\ones^\top$ and using double stochasticity yields $u^{k+1} = u^k - \alpha\ybar^k$. By Lemma~\ref{lem:invariant}, $\ybar^k = \nabla f_\cH(\xbar^k) + \delta^k + \bar{\epsilon}^k$. Substituting,
\begin{equation*}
u^{k+1} = \underbrace{u^k - \alpha\nabla f_\cH(\xbar^k) - \alpha\delta^k}_{A^k} - \alpha\bar{\epsilon}^k.
\end{equation*}
$A^k$ is a function of $x^k$ alone, hence $\cG^k$-measurable, and $\E[\bar{\epsilon}^k | \cG^k] = 0$ by Assumption~\ref{as:stochastic}; the tower property gives $\E[\inner{A^k}{\bar{\epsilon}^k}] = 0$ and
\begin{equation}
\E[\|u^{k+1}\|^2] = \E[\|A^k\|^2] + \alpha^2\E[\|\bar{\epsilon}^k\|^2] \le \E[\|A^k\|^2] + \frac{\alpha^2\sigma^2}{n_\cH}. \label{eq:u_stoch_split}
\end{equation}
We expand
$\|A^k\|^2 = \|u^k - \alpha\nabla f_\cH(\xbar^k)\|^2 - 2\alpha\inner{u^k - \alpha\nabla f_\cH(\xbar^k)}{\delta^k} + \alpha^2\|\delta^k\|^2$.
For the first term, co-coercivity of the $\mu$-strongly convex, $L$-smooth $f_\cH$ gives $\inner{u^k}{\nabla f_\cH(\xbar^k)} \geq \frac{\mu L}{\mu + L}\|u^k\|^2 + \frac{1}{\mu + L}\|\nabla f_\cH(\xbar^k)\|^2$; expanding and using $\alpha \leq \frac{1}{2L} \leq \frac{2}{\mu+L}$ to discard the non-positive gradient-norm term, together with $\frac{2\mu L}{\mu+L} \geq \mu$, yields
\begin{equation}
\|u^k - \alpha\nabla f_\cH(\xbar^k)\|^2 \leq (1 - \alpha\mu)\|u^k\|^2. \label{eq:term1_clean}
\end{equation}
For the cross-term, Young's inequality with parameter $\frac{\mu}{4}$ and~\eqref{eq:term1_clean}:
\begin{equation*}
-2\alpha\inner{u^k - \alpha\nabla f_\cH(\xbar^k)}{\delta^k} \leq \frac{\alpha\mu}{4}\|u^k\|^2 + \frac{4\alpha}{\mu}\|\delta^k\|^2.
\end{equation*}
Combining with $\alpha^2\|\delta^k\|^2 \leq \frac{\alpha}{\mu}\|\delta^k\|^2$ (valid since $\alpha\mu \leq 1$) and~\eqref{eq:delta_bound}:
\begin{equation*}
\|A^k\|^2 \leq \left(1 - \frac{3\alpha\mu}{4}\right)\|u^k\|^2 + \frac{5\alpha}{\mu}\|\delta^k\|^2 \leq \left(1 - \frac{\alpha\mu}{2}\right)\|u^k\|^2 + \frac{5\alpha L^2}{\mu n_\cH}C^k.
\end{equation*}
Taking expectations in~\eqref{eq:u_stoch_split} completes the proof.
\end{proof}

\begin{lemma}[Consensus contraction under $p_b = 0$]
\label{lem:consensus_exact}
Under Assumptions~\ref{as:independence} and~\ref{as:connectivity}:
\begin{equation}
\label{eq:C_bound_specific}
\E[C^{k+1}] \leq \frac{1+\bar{\sigma}_p^2}{2}\, \E[C^k] + \thp\,\alpha^2\, \E[T^k].
\end{equation}
\end{lemma}

\begin{proof}
From~\eqref{eq:X_clean}, using $J_\cH\hat W^k = \hat W^k J_\cH = J_\cH$,
$(I_\cH - J_\cH)\mathbf{X}_\cH^{k+1} = (\hat{W}^k - J_\cH)\mathbf{X}_{\perp,\cH}^k - \alpha\,\mathbf{Y}_{\perp,\cH}^k$.
Young's inequality with parameter $\gamma = \frac{1-\sbp}{2\sbp}$ gives
\begin{equation*}
C^{k+1} \leq (1+\gamma)\Fnorm{(\hat{W}^k - J_\cH)\mathbf{X}_{\perp,\cH}^k}^2 + \Big(1+\frac{1}{\gamma}\Big)\alpha^2 T^k,
\end{equation*}
with $(1+\gamma)\sbp = \frac{1+\sbp}{2}$ and $1+\frac1\gamma = \thp$. The matrix $\mathbf{X}_{\perp,\cH}^k$ is $\cG^k$-measurable, hence $\cF^k$-measurable, and by Assumption~\ref{as:independence} the iteration-$k$ mask is the only randomness beyond $\cF^k$ in the first term; Lemma~\ref{lem:second_moment}(b) and the tower property give
$\E\Fnorm{(\hat{W}^k - J_\cH)\mathbf{X}_{\perp,\cH}^k}^2 \leq \bar{\sigma}_p^2 \E[C^k]$,
which yields~\eqref{eq:C_bound_specific}.
\end{proof}

\begin{lemma}[Stochastic Tracking Contraction]
\label{lem:tracking_stochastic}
Under Assumptions~\ref{as:smooth},~\ref{as:stochastic},~\ref{as:independence}, and~\ref{as:connectivity}, for $\alpha \leq \frac{1}{2L}$ and $p_b = 0$:
\begin{align}
\label{eq:T_bound_stoch}
\begin{split}
\E[T^{k+1}] &\leq \left(\frac{1+\bar{\sigma}_p^2}{2} + 6\alpha^2 L^2\thp\right) \E[T^k] + 30L^2\thp\, \E[C^k] \\
&\quad + 24\alpha^2 n_\cH L^4 \thp\, \E[\|u^k\|^2] + (3 + 12\alpha^2L^2)\thp\, n_\cH \sigma^2 .
\end{split}
\end{align}
\end{lemma}

\begin{proof}
From~\eqref{eq:Y_clean}, $\mathbf{Y}_{\perp,\cH}^{k+1} = (\hat{W}^k - J_\cH)\mathbf{Y}_{\perp,\cH}^k + \Delta\mathbf{G}_\perp^k$ with $\Delta\mathbf{G}_\perp^k = (I_\cH - J_\cH)(\mathbf{G}_\cH^{k+1} - \mathbf{G}_\cH^k)$. Young's inequality with $\eta = \frac{1-\sbp}{2\sbp}$ gives
\begin{equation}
T^{k+1} \le \frac{1+\sbp}{2\sbp}\Fnorm{(\hat{W}^k - J_\cH)\mathbf{Y}_{\perp,\cH}^k}^2 + \thp \Fnorm{\Delta\mathbf{G}_\perp^k}^2 .
\label{eq:T_young}
\end{equation}
The tracker matrix $\mathbf{Y}_{\perp,\cH}^k$ contains the iteration-$k$ samples and is therefore $\cF^k$- but \emph{not} $\cG^k$-measurable; this is precisely why Assumption~\ref{as:independence} conditions the mask on $\cF^k$. Lemma~\ref{lem:second_moment}(b) applied conditionally on $\cF^k$ bounds the expectation of the first term by $\frac{1+\sbp}{2}\E[T^k]$.

For the second term, write $\epsilon_i^k = g_i^k - \nabla f_i(x_i^k)$ and split
$\Delta\mathbf{G}^k_i = \big(g_i^{k+1} - \nabla f_i(x_i^{k+1})\big) + \big(\nabla f_i(x_i^{k+1}) - \nabla f_i(x_i^k)\big) - \epsilon_i^k$.
Conditioned on $\cG^{k+1}$ (which contains the iteration-$k$ mask and samples, hence $x^{k+1}$ and $\epsilon^k$), the first piece is zero-mean with second moment at most $n_\cH\sigma^2$ in Frobenius norm, and is uncorrelated with the remaining two pieces; using $\|a+b\|^2 \le 2\|a\|^2+2\|b\|^2$ on those two pieces, $L$-smoothness, $\Fnorm{\Delta\mathbf G_\perp^k}\le\Fnorm{\Delta\mathbf G^k}$, and $\E\|\epsilon^k_i\|^2\le\sigma^2$:
\begin{equation}
\E\Fnorm{\Delta\mathbf{G}_\perp^k}^2 \le n_\cH\sigma^2 + 2L^2\,\E\sum_{i\in\cH}\|x_i^{k+1}-x_i^k\|^2 + 2n_\cH\sigma^2 .
\label{eq:dG_split}
\end{equation}
Expanding $x_i^{k+1} - x_i^k = \sum_j [\hat{W}^k]_{ij}(x_j^k - x_i^k) - \alpha (y_i^k - \ybar^k) - \alpha \ybar^k$ with $\|a+b+c\|^2 \leq 3\|a\|^2+3\|b\|^2+3\|c\|^2$, Lemma~\ref{lem:deterministic_mixing}, and the decomposition $\ybar^k = \nabla f_\cH(\xbar^k)+\delta^k+\bar\epsilon^k$ with $\|\nabla f_\cH(\xbar^k)\| \le L\|u^k\|$ and~\eqref{eq:delta_bound}:
\begin{equation*}
\sum_{i \in \cH} \|x_i^{k+1} - x_i^k\|^2 \leq (12 + 12\alpha^2 L^2)C^k + 3\alpha^2 T^k + 12\alpha^2 n_\cH L^2 \|u^k\|^2 + 6\alpha^2 n_\cH \|\bar{\epsilon}^k\|^2.
\end{equation*}
With $\alpha \leq \frac{1}{2L}$, $12 + 12\alpha^2 L^2 \leq 15$. Substituting into~\eqref{eq:dG_split} and using $\E\|\bar\epsilon^k\|^2 \le \sigma^2/n_\cH$:
\begin{equation*}
\E\Fnorm{\Delta\mathbf{G}_\perp^k}^2 \leq 30L^2\, \E[C^k] + 6\alpha^2 L^2\, \E[T^k] + 24\alpha^2 n_\cH L^4\, \E\|u^k\|^2 + \left(3 + 12\alpha^2L^2\right)n_\cH\sigma^2.
\end{equation*}
Multiplying by $\thp$ from~\eqref{eq:T_young} and adding the mixing term gives~\eqref{eq:T_bound_stoch}.
\end{proof}

\subsection{Main Result: a Linear System of Inequalities}

Following the stochastic gradient tracking analysis of \cite{pu2021distributed}, we collect the three one-step bounds into a componentwise linear recursion for the state vector
\begin{equation}
\label{eq:zdef}
z^k \;\triangleq\; \big(\E\|u^k\|^2,\ \E[C^k],\ \E[T^k]\big)^\top \in \R^3_{\ge 0}.
\end{equation}

\begin{lemma}[Spectral radius test {\cite[Lem.~5]{pu2021distributed}}]
\label{lem:det_test}
Let $\mathbf{S} \in \R^{3\times3}$ be nonnegative and irreducible with diagonal entries $s_{ii} < \lambda^*$ for some $\lambda^* > 0$. Then $\rho(\mathbf{S}) < \lambda^*$ if and only if $\det(\lambda^*I - \mathbf{S}) > 0$.
\end{lemma}

\begin{theorem}[Convergence of GT-PD under complete isolation]
\label{thm:exact}
Suppose Assumptions~\ref{as:smooth}--\ref{as:connectivity} hold, $p_b = 0$, and the dynamics are the clean dynamics~\eqref{eq:X_clean}--\eqref{eq:Y_clean}. Fix any $\Gamma > 1$ and suppose the step size satisfies
\begin{equation}
\label{eq:stepsize_exact}
\alpha \;\le\; \min\left\{\ \frac{1}{2L},\ \ \frac{1-\sbp}{7L},\ \ \frac{(1-\sbp)^2}{31\sqrt{\Gamma}\,L},\ \ \Big(\tfrac{\Gamma-1}{7680\,\Gamma(\Gamma+1)}\Big)^{1/4}\frac{1-\sbp}{L\sqrt{\kappa}}\ \right\}.
\end{equation}
Then $z^{k+1} \le \mathbf{A}\,z^k + b$ componentwise, where
\begin{equation}
\label{eq:Amatrix}
\mathbf{A} =
\begin{bmatrix}
1-\frac{\alpha\mu}{2} & \frac{5\alpha L^2}{\mu n_\cH} & 0\\[4pt]
0 & \frac{1+\sbp}{2} & \thp\,\alpha^2\\[4pt]
24\alpha^2 n_\cH L^4\thp & 30L^2\thp & \frac{1+\sbp}{2} + 6\alpha^2L^2\thp
\end{bmatrix},
\qquad
b = \begin{bmatrix} \frac{\alpha^2\sigma^2}{n_\cH} \\[2pt] 0 \\[2pt] (3+12\alpha^2L^2)\thp\, n_\cH\sigma^2 \end{bmatrix},
\end{equation}
the spectral radius satisfies $\rho(\mathbf{A}) < 1$, and consequently
\begin{equation}
\label{eq:exact_limsup}
z^k \;\le\; \mathbf{A}^k z^0 + \sum_{l=0}^{k-1}\mathbf{A}^l b,
\qquad
\limsup_{k\to\infty} z^k \;\le\; (I-\mathbf{A})^{-1} b \quad\text{componentwise},
\end{equation}
with linear convergence at rate $\mathcal{O}(\rho(\mathbf{A})^k)$. Moreover, every iterate obeys the uniform componentwise bound
\begin{equation}
\label{eq:sup_bound}
\sup_{k \ge 0}\; z^k \;\le\; (I-\mathbf{A})^{-1}\big(z^0 + b\big).
\end{equation}
\end{theorem}

\begin{proof}
The recursion $z^{k+1}\le \mathbf{A}z^k + b$ is the stack of Lemmas~\ref{lem:opt_exact},~\ref{lem:consensus_exact}, and~\ref{lem:tracking_stochastic}. The matrix $\mathbf{A}$ is nonnegative and irreducible (the entries $a_{12}, a_{23}, a_{31}$ are positive and form the cycle $1\to3\to2\to1$ in the dependency graph). Its diagonal entries are below $1$: $a_{11} = 1-\frac{\alpha\mu}{2}$, $a_{22} = \frac{1+\sbp}{2}$, and $a_{33} \le \frac{3+\sbp}{4}$, the last because the second condition in~\eqref{eq:stepsize_exact} gives $6\alpha^2L^2\thp \le \frac{12\alpha^2L^2}{1-\sbp} \le \frac{1-\sbp}{4}$. By Lemma~\ref{lem:det_test} it remains to verify $\det(I - \mathbf{A}) > 0$, where
\begin{equation}
\label{eq:det_expand}
\det(I-\mathbf{A}) = (1-a_{11})\big[(1-a_{22})(1-a_{33}) - a_{23}a_{32}\big] - a_{12}a_{23}a_{31}.
\end{equation}
Using $1-a_{22} = \frac{1-\sbp}{2}$ and $1-a_{33} \ge \frac{1-\sbp}{4}$:
\emph{(i)} the third condition in~\eqref{eq:stepsize_exact} gives
$a_{23}a_{32} = 30\alpha^2L^2\thp^2 \le \frac{120\,\alpha^2L^2}{(1-\sbp)^2} \le \frac{(1-\sbp)^2}{8\Gamma} \le \frac{1}{\Gamma}(1-a_{22})(1-a_{33})$;
\emph{(ii)} the fourth condition gives

\begin{align*}
a_{12}a_{23}a_{31} 
&= \frac{120\,\alpha^5 L^6 \thp^2}{\mu} \\
&\le \frac{480\,\alpha^5 L^6}{\mu(1-\sbp)^2} \\
&\le \frac{\Gamma-1}{\Gamma(\Gamma+1)}\cdot\frac{\alpha\mu(1-\sbp)^2}{16} \\
&\le \frac{\Gamma-1}{\Gamma(\Gamma+1)}\,(1-a_{11})(1-a_{22})(1-a_{33}),
\end{align*}

where the middle inequality is the fourth condition raised to the fourth power, and the last uses $1-a_{11} = \frac{\alpha\mu}{2}$ and $(1-a_{22})(1-a_{33}) \ge \frac{(1-\sbp)^2}{8}$. Substituting (i) and (ii) into~\eqref{eq:det_expand},
\begin{equation}
\label{eq:det_lb}
\begin{split}
\det(I-\mathbf{A}) &\;\ge\; \Big(1-\tfrac1\Gamma\Big)(1-a_{11})(1-a_{22})(1-a_{33}) - \tfrac{\Gamma-1}{\Gamma(\Gamma+1)}(1-a_{11})(1-a_{22})(1-a_{33})\\
&\;=\; \tfrac{\Gamma-1}{\Gamma+1}\,(1-a_{11})(1-a_{22})(1-a_{33}) \;>\;0.
\end{split}
\end{equation}
Hence $\rho(\mathbf{A})<1$. Iterating the recursion yields the first display in~\eqref{eq:exact_limsup}; since $\mathbf{A}^k \to 0$ linearly and $\sum_l \mathbf{A}^l = (I-\mathbf{A})^{-1}$ entrywise, the limsup bound follows. For~\eqref{eq:sup_bound}, note that all summands in $\sum_{l\ge0}\mathbf{A}^l z^0 = (I-\mathbf{A})^{-1}z^0$ are nonnegative, so each individual term obeys $\mathbf{A}^k z^0 \le (I-\mathbf{A})^{-1}z^0$ componentwise; combined with $\sum_{l<k}\mathbf{A}^l b \le (I-\mathbf{A})^{-1}b$ this gives~\eqref{eq:sup_bound}.
\end{proof}

\begin{corollary}[Componentwise steady states]
\label{cor:exact_steady}
Under the conditions of Theorem~\ref{thm:exact}, with $\kappa = L/\mu$:
\begin{align}
\limsup_{k\to\infty}\,\E\|\xbar^k - x^*\|^2
&\;\le\; \frac{\Gamma+1}{\Gamma}\cdot\frac{2\alpha\sigma^2}{\mu n_\cH}
\;+\; \frac{\Gamma+1}{\Gamma-1}\cdot\frac{1920\,\alpha^2\kappa^2\sigma^2}{(1-\sbp)^4},
\label{eq:u_steady}\\
\limsup_{k\to\infty}\,\E[C^k]
&\;\le\; \frac{\Gamma+1}{\Gamma-1}\cdot\frac{384\,\alpha^2 n_\cH\sigma^2}{(1-\sbp)^4},
\qquad
\limsup_{k\to\infty}\,\E[T^k]
\;\le\; \frac{\Gamma+1}{\Gamma-1}\cdot\frac{96\, n_\cH\sigma^2}{(1-\sbp)^2}.
\label{eq:CT_steady}
\end{align}
\end{corollary}
\begin{proof}
Write $M = I - \mathbf{A}$ and let $\mathrm{cof}_{ij}$ denote its cofactors, so $[(I-\mathbf{A})^{-1}b]_j = \frac{1}{\det M}\sum_i \mathrm{cof}_{ij}\, b_i$. With the zero pattern of $\mathbf{A}$,
\begin{align*}
[(I-\mathbf{A})^{-1}b]_1 &= \tfrac{1}{\det M}\Big(\big[(1-a_{22})(1-a_{33})-a_{23}a_{32}\big]b_1 + a_{12}a_{23}\,b_3\Big),\\
[(I-\mathbf{A})^{-1}b]_2 &= \tfrac{1}{\det M}\Big(a_{23}a_{31}\,b_1 + a_{23}(1-a_{11})\,b_3\Big),\\
[(I-\mathbf{A})^{-1}b]_3 &= \tfrac{1}{\det M}\Big(a_{31}(1-a_{22})\,b_1 + (1-a_{11})(1-a_{22})\,b_3\Big).
\end{align*}
Use~\eqref{eq:det_lb} together with $\det M \ge \frac{\Gamma}{\Gamma+1}(1-a_{11})[(1-a_{22})(1-a_{33})-a_{23}a_{32}]$ (which follows from~\eqref{eq:det_expand} and step (ii) in the proof of Theorem~\ref{thm:exact}), the bounds $1-a_{11} = \frac{\alpha\mu}{2}$, $(1-a_{22})(1-a_{33}) \ge \frac{(1-\sbp)^2}{8}$, $\thp \le \frac{2}{1-\sbp}$, and $b_3 \le 6\thp n_\cH\sigma^2 \le \frac{12 n_\cH\sigma^2}{1-\sbp}$ (since $12\alpha^2L^2 \le 3$). For the first component:

\begin{align*}
[(I-\mathbf{A})^{-1}b]_1 
&\le \frac{\Gamma+1}{\Gamma}\cdot\frac{b_1}{1-a_{11}} + \frac{\Gamma+1}{\Gamma-1}\cdot\frac{a_{12}a_{23}\,b_3}{(1-a_{11})(1-a_{22})(1-a_{33})} \\
&\le \frac{\Gamma+1}{\Gamma}\cdot\frac{2\alpha\sigma^2}{\mu n_\cH} + \frac{\Gamma+1}{\Gamma-1}\cdot\frac{1920\,\alpha^2 L^2\sigma^2}{\mu^2(1-\sbp)^4},
\end{align*}

where the last step substitutes $a_{12}a_{23} = \frac{5\alpha^3L^2\thp}{\mu n_\cH} \le \frac{10\alpha^3L^2}{\mu n_\cH(1-\sbp)}$. For the second component, the term $a_{23}(1-a_{11})b_3$ dominates: $a_{23}b_3/((1-a_{22})(1-a_{33})) \le \thp\alpha^2\cdot\frac{12n_\cH\sigma^2}{1-\sbp}\cdot\frac{8}{(1-\sbp)^2} \le \frac{192\alpha^2 n_\cH\sigma^2}{(1-\sbp)^4}$, while the term $a_{23}a_{31}b_1/(1-a_{11})$ is bounded by the same quantity under the fourth condition of~\eqref{eq:stepsize_exact} (indeed $\frac{a_{31}b_1}{(1-a_{11})b_3} = \frac{16\,\alpha^3L^4}{\mu n_\cH} = \frac{16\,\kappa(\alpha L)^3}{n_\cH} \le 1$, since the fourth condition gives $(\alpha L)^3 \le \kappa^{-3/2}(1-\sbp)^3/16$); doubling the dominant term yields~\eqref{eq:CT_steady} for $C$. For the third component, $(1-a_{11})(1-a_{22})b_3/\det M \le \frac{\Gamma+1}{\Gamma-1}\cdot\frac{b_3}{1-a_{33}} \le \frac{\Gamma+1}{\Gamma-1}\cdot\frac{48 n_\cH\sigma^2}{(1-\sbp)^2}$, and the $b_1$ term is dominated by it by the same argument; doubling gives the bound for $T$.
\end{proof}

\begin{remark}[Main takeaways for Theorem~\ref{thm:exact}]
\label{rem:exact_takeaways}
Three features deserve emphasis. First, the optimality-gap neighborhood~\eqref{eq:u_steady} \emph{vanishes with the step size} at the correct orders: the leading term $\mathcal{O}(\alpha\sigma^2/\mu n_\cH)$ matches centralized stochastic gradient descent, including the $1/n_\cH$ variance reduction, and the network-dependent term is $\mathcal{O}(\alpha^2)$, mirroring the structure of stochastic gradient tracking without adversaries \cite{pu2021distributed}. Second, the tracking floor in~\eqref{eq:CT_steady} is $\Theta(n_\cH\sigma^2)$ \emph{independently of $\alpha$}: the tracker re-injects two fresh layers of sampling noise at every iteration, so tracking disagreement does not vanish under constant step size. This floor is not an artifact of the analysis — it is what drives the trust-score behavior of honest edges at stationarity (Corollary~\ref{cor:S0_rule}) and what makes the threshold gate in Definition~\ref{def:data_driven_p} necessary. Third, all step-size conditions in~\eqref{eq:stepsize_exact} are dimensionally homogeneous (each right-hand side scales as $1/L$) and explicit, with one free parameter $\Gamma$ trading the two constants in~\eqref{eq:u_steady}. An explicit linear rate bound of the form $\rho(\mathbf{A}) \le 1 - \Theta(\alpha\mu)$ can be extracted exactly as in \cite[Cor.~1]{pu2021distributed}; we do not repeat the argument.
\end{remark}

\subsection{Honest Non-Clipping under Complete Isolation}

We now show that the projection radius can be chosen, from problem data alone, so that the projected algorithm coincides with the clean dynamics over any fixed horizon with high probability. Define
\begin{equation}
\label{eq:CThat}
\widehat{C} \;\triangleq\; \big[(I-\mathbf{A})^{-1}(z^0 + b)\big]_2,
\qquad
\widehat{T} \;\triangleq\; \big[(I-\mathbf{A})^{-1}(z^0 + b)\big]_3,
\end{equation}
with $\mathbf{A}, b$ as in~\eqref{eq:Amatrix} and $z^0 = \big(\|u^0\|^2,\, C^0,\, \E[T^0]\big)^\top$. These are computable from the problem parameters: under common spatial initialization $C^0 = 0$, and $\E[T^0] \le 2n_\cH\sigma^2 + 2\sum_{i\in\cH}\|\nabla f_i(x_0) - \nabla f_\cH(x_0)\|^2$ by Assumption~\ref{as:stochastic}.

\begin{proposition}[Honest non-clipping under complete isolation]
\label{prop:honest_nonclip}
Suppose $p_b = 0$, all honest agents share a common spatial initialization $x_i^0 = x_0$ (so $C^0 = 0$), and the step size satisfies~\eqref{eq:stepsize_exact}. Fix a horizon $K \ge 1$ and a confidence level $\epsilon \in (0,1)$. If the projection radius satisfies
\begin{equation}
\label{eq:tau_nonclip}
\tau \;\ge\; 2\,\sqrt{\frac{2(K+1)\,\max\{\widehat{C},\,\widehat{T}\}}{\epsilon}},
\end{equation}
then with probability at least $1-\epsilon$ the self-centered projection does not modify any honest message for all $k = 0,1,\ldots,K$; on this event the GT-PD trajectory coincides with the clean dynamics~\eqref{eq:X_clean}--\eqref{eq:Y_clean}, and the guarantees of Theorem~\ref{thm:exact} and Corollary~\ref{cor:exact_steady} apply to it verbatim.
\end{proposition}

\begin{proof}
\emph{Coupling construction.} Realize the dropout masks through i.i.d.\ uniform random variables: edge $(i,j)$ at iteration $k$ is retained iff $U^k_{ij} \le p^k_{ij}$, where $\{U^k_{ij}\}$ are shared by the two systems below, as are the oracle samples $\{\xi_i^k\}$. The \emph{virtual system} $(\mathbf X^k_{\mathrm{ref}}, \mathbf Y^k_{\mathrm{ref}})$ runs the clean dynamics~\eqref{eq:X_clean}--\eqref{eq:Y_clean} (no projection applied to any message) from the same initialization; its retention probabilities are computed from its own trajectory. The virtual system satisfies the hypotheses of Theorem~\ref{thm:exact} unconditionally, so by~\eqref{eq:sup_bound},
\begin{equation}
\label{eq:ref_sup}
\sup_{k\ge0}\ \E[C^k_{\mathrm{ref}}] \le \widehat C, \qquad \sup_{k\ge0}\ \E[T^k_{\mathrm{ref}}] \le \widehat T .
\end{equation}

\emph{Markov and union bound.} Let $\delta = \frac{\epsilon}{2(K+1)}$ and define
$\mathcal{G} \triangleq \bigcap_{k=0}^{K}\big\{C^k_{\mathrm{ref}} \le \widehat C/\delta\big\}\cap\big\{T^k_{\mathrm{ref}} \le \widehat T/\delta\big\}$.
By Markov's inequality and~\eqref{eq:ref_sup}, each of the $2(K+1)$ events fails with probability at most $\delta$, so $\Prob(\mathcal{G}) \ge 1-\epsilon$. On $\mathcal{G}$, for any honest pair $i,j$ and $k \le K$,
\begin{equation*}
\|x^k_{\mathrm{ref},j} - x^k_{\mathrm{ref},i}\|^2 \le 2\|x^k_{\mathrm{ref},j}-\xbar^k_{\mathrm{ref}}\|^2 + 2\|x^k_{\mathrm{ref},i}-\xbar^k_{\mathrm{ref}}\|^2 \le 4C^k_{\mathrm{ref}} \le \frac{8(K+1)\widehat C}{\epsilon} \le \tau^2,
\end{equation*}
by~\eqref{eq:tau_nonclip}, and identically $\|y^k_{\mathrm{ref},j} - y^k_{\mathrm{ref},i}\| \le \tau$.

\emph{Induction.} We claim that on $\mathcal{G}$ the real (projected) and virtual systems coincide through iteration $K$. At $k=0$ they share the initialization. Suppose they coincide through iteration $k \le K$. Then their retention probabilities at iteration $k$ agree (they are functions of the common trajectory), hence under the coupling their masks agree; and on $\mathcal{G}$ every honest pairwise deviation at iteration $k$ is at most $\tau$, so $\cP_{\tau,i}$ acts as the identity on every honest message of the real system. With $p_b=0$ no Byzantine message is used. The two updates therefore coincide, and the trajectories agree at $k+1$. This proves the claim, and with it the proposition.
\end{proof}

\begin{remark}
\label{rem:universal_projection}
The radius~\eqref{eq:tau_nonclip} scales as $\mathcal{O}(\sqrt{(K+1)/\epsilon})$: the horizon factor is the price of a union bound over the trajectory, and $\epsilon^{-1/2}$ the price of the Markov tail. Crucially, under $p_b = 0$ the quantities $\widehat C, \widehat T$ do not depend on $\tau$, so~\eqref{eq:tau_nonclip} is non-circular: $\tau$ is computable from the problem parameters and the initialization before running the algorithm. Sharper horizon dependence (e.g., via maximal inequalities in place of the union bound) is possible but not pursued here.
\end{remark}


\section{GT-PD-L: Convergence under Partial Isolation (\texorpdfstring{$p_b > 0$}{p\_b > 0})}
\label{sec:convergence_leaky}

When Byzantine edges are retained with positive probability ($p_b > 0$), the universal self-centered projection bounds the magnitude of each individual adversarial message (Proposition~\ref{prop:enforced_bound}), but the resulting perturbations $\mathbf{E}_{x,\cB}^k$ and $\mathbf{E}_{y,\cB}^k$ persist across iterations. Under the standard tracking update, the perturbation to the gradient tracker causes the tracking-invariant violation $s^k = \ybar^k - \frac{1}{n_\cH}\sum_i g_i^k$ to accumulate without contraction (Remark~\ref{rem:tracking_invariant}). To prevent this unbounded drift without discrete epoch resets, GT-PD-L employs a continuous leaky integrator in the tracking update, with decay parameter $\beta \in (0, 1/2]$:
\begin{align}
y_i^{k+1} &= (1-\beta) \sum_{j=1}^{n} [W^k]_{ij}\, \cP_{\tau,i}(\hat{y}_j^k) + g_i^{k+1} - (1-\beta)g_i^k. \label{eq:y_leaky}
\end{align}

Two of the three one-step bounds below hold \emph{unconditionally} — regardless of whether honest messages are clipped — because, by Lemma~\ref{lem:honest_clip}, honest clipping never perturbs network averages. Only the disagreement contractions (Lemmas~\ref{lem:consensus_leaky} and~\ref{lem:tracking_leaky}) require honest messages to pass unclipped; Proposition~\ref{prop:honest_nonclip_partial} establishes that this holds with high probability over any fixed horizon, under an explicit budget on the aggregate Byzantine weight.

\subsection{Unconditional Bounds: Tracking Invariant and Optimality}

\begin{lemma}[Tracking-invariant violation under GT-PD-L]
\label{lem:s_bound}
Under GT-PD-L with $y^0_i = g^0_i$, for every realization and every $k \ge 0$ — with or without honest clipping —
\begin{equation}
\label{eq:s_recursion}
s^{k+1} = (1-\beta)\,s^k + (1-\beta)\,\bar e^k_{y,\cB},
\qquad
\|s^k\| \;\le\; \frac{(1-\beta)\,\zeta}{\beta},
\qquad
\|s^k\|^2 \;\le\; \bar s^2 \triangleq \frac{(1-\beta)^2\zeta^2}{\beta^2}.
\end{equation}
\end{lemma}
\begin{proof}
Averaging~\eqref{eq:y_leaky} over $i \in \cH$ and using $\ones^\top\hat W^k = \ones^\top$,
$\ybar^{k+1} = (1-\beta)\ybar^k + (1-\beta)\big(\bar e^k_{y,\textup{clip}} + \bar e^k_{y,\cB}\big) + \bar g^{k+1} - (1-\beta)\bar g^k$.
By Lemma~\ref{lem:honest_clip}, $\bar e^k_{y,\textup{clip}} = 0$ identically. Subtracting $\bar g^{k+1}$ gives the recursion in~\eqref{eq:s_recursion}. Since $\|\bar e^k_{y,\cB}\| \le \zeta$ (Proposition~\ref{prop:enforced_bound}) and $s^0 = 0$,
$\|s^k\| \le (1-\beta)\zeta\sum_{j=0}^{k-1}(1-\beta)^{j} \le (1-\beta)\zeta/\beta$.
\end{proof}

\begin{remark}
\label{rem:no_coupling}
Lemma~\ref{lem:s_bound} is where exact average preservation earns its keep. If the honest-clipping average were merely \emph{bounded} (e.g., $\|\bar e^k_{y,\textup{clip}}\|^2 \le 4T^k/n_\cH$ by Cauchy--Schwarz from Lemma~\ref{lem:honest_clip}), the violation $s^k$ would inherit a dependence on the past tracking disagreement, creating a feedback loop from $T^k$ into the optimality gap that would have to be controlled by an additional coupling condition. An earlier version of this work followed that route and paid for it with a coupling parameter multiplying the full state. The cancellation~\eqref{eq:clip_cancel} shows the coupling is identically zero: $s^k$ is driven by the Byzantine channel alone, and remains bounded by the deterministic constant $\bar s$ for every realization.
\end{remark}

\begin{lemma}[Leaky optimality contraction, unconditional]
\label{lem:opt_leaky}
Under Assumptions~\ref{as:smooth},~\ref{as:strong_cvx}, and~\ref{as:stochastic}, with $\zeta = \delta^{\cB}_{\max}\tau$ and $\alpha \le \frac{1}{2L}$, GT-PD-L satisfies — with or without honest clipping, and without any independence assumption between the round-$k$ Byzantine messages and the round-$k$ oracle samples —
\begin{equation}
\label{eq:opt_leaky}
\E[\|u^{k+1}\|^2] \;\le\; \Big(1 - \frac{\alpha\mu}{4}\Big)\E[\|u^k\|^2] + \frac{8\alpha L^2}{\mu n_\cH}\,\E[C^k] + \frac{8\alpha}{\mu}\,\bar s^2 + \frac{9\,\zeta^2}{\alpha\mu} + \frac{2\alpha^2\sigma^2}{n_\cH}.
\end{equation}
\end{lemma}
\begin{proof}
Averaging the $x$-update and using double stochasticity and Lemma~\ref{lem:honest_clip} ($\bar e^k_{x,\textup{clip}} = 0$),
$u^{k+1} = u^k - \alpha\ybar^k + \bar e^k_{x,\cB}$.
Decompose $\ybar^k = \nabla f_\cH(\xbar^k) + \delta^k + s^k + \bar\epsilon^k$ (definition of $s^k$, $\delta^k$, $\bar\epsilon^k$), so that
\begin{equation*}
u^{k+1} \;=\; \underbrace{u^k - \alpha\nabla f_\cH(\xbar^k)}_{A_1} \;+\; \underbrace{(-\alpha\delta^k - \alpha s^k)}_{A_2} \;-\; \alpha\bar\epsilon^k \;+\; \bar e^k_{x,\cB} .
\end{equation*}
The adversarial term $\bar e^k_{x,\cB}$ may be correlated with the same-round noise $\bar\epsilon^k$ (the adversary observes the exchanged trackers, which contain $g^k$); we therefore split it off by Young's inequality with parameter $\frac{\alpha\mu}{8}$, using only $\|\bar e^k_{x,\cB}\|\le\zeta$:
\begin{equation*}
\|u^{k+1}\|^2 \le \Big(1+\frac{\alpha\mu}{8}\Big)\|A_1 + A_2 - \alpha\bar\epsilon^k\|^2 + \Big(1+\frac{8}{\alpha\mu}\Big)\zeta^2
\le \Big(1+\frac{\alpha\mu}{8}\Big)\|A_1 + A_2 - \alpha\bar\epsilon^k\|^2 + \frac{9\zeta^2}{\alpha\mu},
\end{equation*}
since $\alpha\mu \le 1$. Now $A_1$ and $A_2$ are $\cG^k$-measurable: $A_1$ and $\delta^k$ are functions of $x^k$, and $s^k$ is determined by quantities realized through round $k-1$ (Lemma~\ref{lem:s_bound} and Remark~\ref{rem:filtration}). Hence $\E[\inner{A_1+A_2}{\bar\epsilon^k}] = 0$ by Assumption~\ref{as:stochastic} and
$\E\|A_1 + A_2 - \alpha\bar\epsilon^k\|^2 = \E\|A_1+A_2\|^2 + \alpha^2\E\|\bar\epsilon^k\|^2 \le \E\|A_1+A_2\|^2 + \frac{\alpha^2\sigma^2}{n_\cH}$.
Young's inequality with parameter $\frac{\alpha\mu}{2}$, the co-coercivity bound~\eqref{eq:term1_clean}, and $1 + \frac{2}{\alpha\mu} \le \frac{3}{\alpha\mu}$ give
\begin{equation*}
\|A_1+A_2\|^2 \le \Big(1+\frac{\alpha\mu}{2}\Big)(1-\alpha\mu)\|u^k\|^2 + \frac{3}{\alpha\mu}\|A_2\|^2
\le \Big(1-\frac{\alpha\mu}{2}\Big)\|u^k\|^2 + \frac{6\alpha}{\mu}\big(\|\delta^k\|^2 + \|s^k\|^2\big),
\end{equation*}
using $\|A_2\|^2 \le 2\alpha^2(\|\delta^k\|^2 + \|s^k\|^2)$. Collecting, with $(1+\frac{\alpha\mu}{8})(1-\frac{\alpha\mu}{2}) \le 1-\frac{\alpha\mu}{4}$ and $(1+\frac{\alpha\mu}{8}) \le \frac98$:
\begin{equation*}
\E\|u^{k+1}\|^2 \le \Big(1-\frac{\alpha\mu}{4}\Big)\E\|u^k\|^2 + \frac{8\alpha}{\mu}\big(\E\|\delta^k\|^2 + \bar s^2\big) + \frac{2\alpha^2\sigma^2}{n_\cH} + \frac{9\zeta^2}{\alpha\mu}.
\end{equation*}
Substituting~\eqref{eq:delta_bound} and Lemma~\ref{lem:s_bound} completes the proof.
\end{proof}

\subsection{Disagreement Bounds on the Non-Clipping Event}

The next two lemmas analyze the \emph{Byzantine-only dynamics}, in which honest messages pass through the projection unmodified ($\mathbf{E}^k_{x,\textup{clip}} = \mathbf{E}^k_{y,\textup{clip}} = 0$) while Byzantine rows remain bounded by $\zeta$. Proposition~\ref{prop:honest_nonclip_partial} shows the actual trajectory coincides with these dynamics over any fixed horizon with high probability.

\begin{lemma}[Consensus contraction under partial isolation]
\label{lem:consensus_leaky}
Under Assumptions~\ref{as:independence} and~\ref{as:connectivity}, the Byzantine-only dynamics satisfy
\begin{equation}
\label{eq:C_leaky}
\E[C^{k+1}] \;\le\; \frac{1+\sbp}{2}\,\E[C^k] + 2\thp\,\alpha^2\,\E[T^k] + 2\thp\, n_\cH\,\zeta^2 .
\end{equation}
\end{lemma}
\begin{proof}
Projecting the $x$-update onto the disagreement subspace,
$\mathbf X^{k+1}_{\perp,\cH} = (\hat W^k - J_\cH)\mathbf X^k_{\perp,\cH} + \big({-\alpha}\mathbf Y^k_{\perp,\cH} + \mathbf E^k_{x,\cB,\perp}\big)$.
Young's inequality with $\gamma = \frac{1-\sbp}{2\sbp}$ on the two groups, then $\|a+b\|^2\le2\|a\|^2+2\|b\|^2$ inside the second group:
\begin{equation*}
C^{k+1} \le (1+\gamma)\Fnorm{(\hat W^k - J_\cH)\mathbf X^k_{\perp,\cH}}^2 + \thp\cdot 2\big(\alpha^2 T^k + \Fnorm{\mathbf E^k_{x,\cB,\perp}}^2\big).
\end{equation*}
The first term is handled exactly as in Lemma~\ref{lem:consensus_exact} (its expectation is at most $\frac{1+\sbp}{2}\E C^k$), and $\Fnorm{\mathbf E^k_{x,\cB,\perp}}^2 \le \Fnorm{\mathbf E^k_{x,\cB}}^2 \le n_\cH\zeta^2$ since $I_\cH - J_\cH$ is an orthogonal projection.
\end{proof}

\begin{lemma}[Leaky tracking contraction under partial isolation]
\label{lem:tracking_leaky}
Under Assumptions~\ref{as:smooth},~\ref{as:stochastic},~\ref{as:independence}, and~\ref{as:connectivity}, with $\alpha \le \frac{1}{2L}$ and $\beta\in(0,1/2]$, the Byzantine-only GT-PD-L dynamics satisfy
\begin{align}
\label{eq:T_leaky}
\E[T^{k+1}] &\le \Big(\frac{1+\sbp}{2} + 24\thp\alpha^2L^2\Big)\E[T^k]
+ (120 + 24\beta^2)\thp L^2\,\E[C^k]
+ \big(96\,\alpha^2 L^4 + 24\,\beta^2 L^2\big)\thp\, n_\cH\,\E\|u^k\|^2 \nonumber\\
&\qquad + 3\thp\Big(12\, n_\cH\sigma^2 + 8\, n_\cH \bar s^2 + 8\, n_\cH L^2\zeta^2 + n_\cH\,\zeta^2 + 4\,\beta^2\zeta_f^2\Big).
\end{align}
\end{lemma}
\begin{proof}
Projecting~\eqref{eq:y_leaky} onto the disagreement subspace and splitting $\mathbf G^{k+1}_\cH - (1-\beta)\mathbf G^k_\cH = \Delta\mathbf G^k + \beta\mathbf G^k_\cH$:
\begin{equation*}
\mathbf Y^{k+1}_{\perp,\cH} = (1-\beta)(\hat W^k - J_\cH)\mathbf Y^k_{\perp,\cH} + \Big(\Delta\mathbf G^k_\perp + \beta\,\mathbf G^k_{\cH,\perp} + (1-\beta)\mathbf E^k_{y,\cB,\perp}\Big).
\end{equation*}
Young's inequality with $\eta = \frac{1-\sbp}{2\sbp}$ on the mixing term versus the bracketed group, then $\|a+b+c\|^2 \le 3(\|a\|^2+\|b\|^2+\|c\|^2)$ inside the group, $(1-\beta)^2 \le 1$ throughout, and Lemma~\ref{lem:second_moment}(b) conditionally on $\cF^k$ for the mixing term:
\begin{equation}
\E[T^{k+1}] \le \frac{1+\sbp}{2}\E[T^k] + 3\thp\Big(\E\Fnorm{\Delta\mathbf G^k_\perp}^2 + \beta^2\,\E\Fnorm{\mathbf G^k_\cH}^2 + n_\cH\zeta^2\Big).
\label{eq:T_leaky_young}
\end{equation}
\emph{Gradient-difference term.} As in Lemma~\ref{lem:tracking_stochastic} but with the Byzantine rows in the $x$-update,
$x_i^{k+1}-x_i^k = \sum_j[\hat W^k]_{ij}(x^k_j - x^k_i) - \alpha(y^k_i - \ybar^k) - \alpha\ybar^k + [\mathbf E^k_{x,\cB}]_i$,
and $\ybar^k = \nabla f_\cH(\xbar^k) + \delta^k + s^k + \bar\epsilon^k$; applying $\|\sum_{r=1}^4 v_r\|^2 \le 4\sum_r\|v_r\|^2$ twice, Lemma~\ref{lem:deterministic_mixing},~\eqref{eq:delta_bound}, $\|\nabla f_\cH(\xbar^k)\|\le L\|u^k\|$, and Lemma~\ref{lem:s_bound}:
\begin{equation*}
\sum_{i\in\cH}\|x^{k+1}_i - x^k_i\|^2 \le (16+16\alpha^2L^2)C^k + 4\alpha^2T^k + 16\alpha^2 n_\cH L^2\|u^k\|^2 + 16\alpha^2 n_\cH \bar s^2 + 16\alpha^2 n_\cH\|\bar\epsilon^k\|^2 + 4n_\cH\zeta^2 .
\end{equation*}
With $\alpha\le\frac{1}{2L}$: $16+16\alpha^2L^2 \le 20$, $32\alpha^2 L^2 \le 8$. The split~\eqref{eq:dG_split} applies verbatim (the zero-mean step conditions on $\cG^{k+1}$, which contains the round-$k$ adversarial messages by Remark~\ref{rem:filtration}, so no independence from the adversary is needed), giving
\begin{equation*}
\E\Fnorm{\Delta\mathbf G^k_\perp}^2 \le 40L^2\,\E C^k + 8\alpha^2L^2\,\E T^k + 32\alpha^2 n_\cH L^4\,\E\|u^k\|^2 + 8 n_\cH\bar s^2 + 11\, n_\cH\sigma^2 + 8 n_\cH L^2\zeta^2 ,
\end{equation*}
where we used $32\alpha^2 n_\cH L^2\bar s^2 \le 8n_\cH\bar s^2$, $2L^2\cdot16\alpha^2 n_\cH\E\|\bar\epsilon^k\|^2 \le 8\sigma^2 \le 8n_\cH\sigma^2$ folded as $3n_\cH\sigma^2 + 8n_\cH\sigma^2 \le 11n_\cH\sigma^2$, and $2L^2\cdot 4n_\cH\zeta^2 = 8n_\cH L^2\zeta^2$.
\emph{Leakage term.} Decomposing $g^k_i = \nabla f_i(x^k_i) + \epsilon^k_i$, $\|\nabla f_i(x^k_i)\|^2 \le 2L^2\|x^k_i - x^*\|^2 + 2\|\nabla f_i(x^*)\|^2$, and $\sum_i\|x^k_i - x^*\|^2 \le 2C^k + 2n_\cH\|u^k\|^2$:
\begin{equation*}
\beta^2\,\E\Fnorm{\mathbf G^k_\cH}^2 \le \beta^2\big(8n_\cH L^2\,\E\|u^k\|^2 + 8L^2\,\E C^k + 4\zeta_f^2 + 2n_\cH\sigma^2\big).
\end{equation*}
Substituting both displays into~\eqref{eq:T_leaky_young}, collecting coefficients ($3\thp\cdot8\alpha^2L^2 = 24\thp\alpha^2L^2$ on $T$; $3\thp(40+8\beta^2)L^2 = (120+24\beta^2)\thp L^2$ on $C$; $3\thp(32\alpha^2L^4 + 8\beta^2L^2)n_\cH$ on $\|u\|^2$), and bounding $11 + 2\beta^2 \le 12$ for $\beta \le 1/2$ in the constant, yields~\eqref{eq:T_leaky}.
\end{proof}

\subsection{Main Result: Convergence of GT-PD-L}

\begin{theorem}[Convergence of GT-PD-L under partial isolation]
\label{thm:leaky}
Suppose Assumptions~\ref{as:smooth}--\ref{as:connectivity} hold with $p_b > 0$, GT-PD-L runs with leak parameter $\beta \in (0,1/2]$ and projection radius $\tau > 0$, and the trajectory follows the Byzantine-only dynamics (no honest message is clipped; Proposition~\ref{prop:honest_nonclip_partial} provides this event with high probability). Fix any $\Gamma > 1$ and suppose
\begin{equation}
\label{eq:stepsize_leaky}
\begin{gathered}
\alpha \;\le\; \min\left\{\frac{1}{2L},\ \frac{1-\sbp}{14L},\ \frac{(1-\sbp)^2}{90\sqrt{\Gamma}L},\ \Big(\tfrac{\Gamma-1}{393216\,\Gamma(\Gamma+1)}\Big)^{1/4}\frac{1-\sbp}{L\sqrt{\kappa}}\right\},\\
\alpha\beta \;\le\; \Big(\tfrac{\Gamma-1}{98304\,\Gamma(\Gamma+1)}\Big)^{1/2}\frac{(1-\sbp)^2\,\mu}{L^2}.
\end{gathered}
\end{equation}
Then $z^{k+1} \le \mathbf{A}'z^k + b'$ componentwise, with

\begin{equation}
\label{eq:Amatrix_leaky}
\begin{split}
\mathbf{A}' &=
\begin{bmatrix}
1-\frac{\alpha\mu}{4} & \frac{8\alpha L^2}{\mu n_\cH} & 0\\[6pt]
0 & \frac{1+\sbp}{2} & 2\thp\alpha^2\\[6pt]
\big(96\alpha^2L^4 + 24\beta^2L^2\big)\thp n_\cH & (120+24\beta^2)\thp L^2 & \frac{1+\sbp}{2} + 24\thp\alpha^2L^2
\end{bmatrix}, \\[10pt]
b' &= 
\begin{bmatrix} 
\frac{2\alpha^2\sigma^2}{n_\cH} + \frac{9\zeta^2}{\alpha\mu} + \frac{8\alpha}{\mu}\bar s^2 \\[6pt] 
2\thp n_\cH\zeta^2 \\[6pt] 
3\thp\big(12n_\cH\sigma^2 + 8n_\cH\bar s^2 + 8n_\cH L^2\zeta^2 + n_\cH\zeta^2 + 4\beta^2\zeta_f^2\big) 
\end{bmatrix}.
\end{split}
\end{equation}

where $\bar s^2 = (1-\beta)^2\zeta^2/\beta^2$. The spectral radius satisfies $\rho(\mathbf{A}') < 1$, and
\begin{equation}
\label{eq:leaky_limsup}
\limsup_{k\to\infty} z^k \le (I-\mathbf{A}')^{-1}b' \quad\text{componentwise},
\qquad
\sup_{k\ge0} z^k \le (I-\mathbf{A}')^{-1}(z^0 + b'),
\end{equation}
with linear convergence at rate $\mathcal{O}(\rho(\mathbf{A}')^k)$.
\end{theorem}

\begin{proof}
The recursion is the stack of Lemmas~\ref{lem:opt_leaky},~\ref{lem:consensus_leaky}, and~\ref{lem:tracking_leaky}. $\mathbf{A}'$ is nonnegative and irreducible ($a'_{12}, a'_{23}, a'_{31} > 0$ form a cycle). Diagonals: $a'_{11} = 1 - \frac{\alpha\mu}{4} < 1$; $a'_{22} = \frac{1+\sbp}{2}$; and the second condition in~\eqref{eq:stepsize_leaky} gives $24\thp\alpha^2L^2 \le \frac{48\alpha^2L^2}{1-\sbp} \le \frac{1-\sbp}{4}$, so $a'_{33} \le \frac{3+\sbp}{4} < 1$. By Lemma~\ref{lem:det_test} it suffices to show $\det(I-\mathbf{A}') > 0$, with the expansion~\eqref{eq:det_expand} applied to $\mathbf{A}'$.

\emph{(i)} Third condition in~\eqref{eq:stepsize_leaky}: using $\beta\le\frac12$, $a'_{23}a'_{32} = 2\thp\alpha^2\cdot(120+24\beta^2)\thp L^2 \le 252\,\thp^2\alpha^2L^2 \le \frac{1008\,\alpha^2L^2}{(1-\sbp)^2} \le \frac{(1-\sbp)^2}{8\Gamma} \le \frac{1}{\Gamma}(1-a'_{22})(1-a'_{33})$.

\emph{(ii)} For the triple product,
\begin{equation*}
a'_{12}a'_{23}a'_{31} = \frac{8\alpha L^2}{\mu n_\cH}\cdot 2\thp\alpha^2\cdot\big(96\alpha^2L^4+24\beta^2L^2\big)\thp n_\cH
= \frac{1536\,\thp^2\alpha^5L^6}{\mu} + \frac{384\,\thp^2\alpha^3\beta^2L^4}{\mu}.
\end{equation*}
The fourth condition in~\eqref{eq:stepsize_leaky}, raised to the fourth power and combined with $\thp^2 \le 4/(1-\sbp)^2$, bounds the first summand by $\frac{\Gamma-1}{\Gamma(\Gamma+1)}\cdot\frac{\alpha\mu(1-\sbp)^2}{64}$; the $\alpha\beta$ condition, squared, bounds the second summand by the same quantity. Since $(1-a'_{11})(1-a'_{22})(1-a'_{33}) \ge \frac{\alpha\mu}{4}\cdot\frac{(1-\sbp)^2}{8} = \frac{\alpha\mu(1-\sbp)^2}{32}$, their sum satisfies
$a'_{12}a'_{23}a'_{31} \le \frac{\Gamma-1}{\Gamma(\Gamma+1)}\cdot\frac{\alpha\mu(1-\sbp)^2}{32} \le \frac{\Gamma-1}{\Gamma(\Gamma+1)}(1-a'_{11})(1-a'_{22})(1-a'_{33})$.
Substituting (i) and (ii) into the determinant expansion gives, exactly as in~\eqref{eq:det_lb},
$\det(I-\mathbf{A}') \ge \frac{\Gamma-1}{\Gamma+1}(1-a'_{11})(1-a'_{22})(1-a'_{33}) > 0$,
hence $\rho(\mathbf{A}') < 1$. The limsup and uniform bounds follow as in Theorem~\ref{thm:exact}.
\end{proof}

\begin{corollary}[Steady state of GT-PD-L]
\label{cor:leaky_steady}
Under the conditions of Theorem~\ref{thm:leaky},
\begin{align}
\limsup_{k\to\infty}\E\|\xbar^k - x^*\|^2
\;\le\;& \frac{\Gamma+1}{\Gamma}\left[\frac{8\alpha\sigma^2}{\mu n_\cH} + \frac{36\,\zeta^2}{\alpha^2\mu^2} + \frac{32}{\mu^2}\cdot\frac{(1-\beta)^2\zeta^2}{\beta^2}\right]
+ \frac{\Gamma+1}{\Gamma-1}\cdot\frac{256\,\kappa^2\zeta^2}{(1-\sbp)^2} \nonumber\\
&+ \frac{\Gamma+1}{\Gamma-1}\cdot\frac{6144\,\alpha^2\kappa^2}{(1-\sbp)^4}\left[12\sigma^2 + \frac{8(1-\beta)^2\zeta^2}{\beta^2} + 8L^2\zeta^2 + \zeta^2 + \frac{4\beta^2\zeta_f^2}{n_\cH}\right].
\label{eq:leaky_steady}
\end{align}
In particular, $\limsup_k \E\|\xbar^k - x^*\|^2 = \mathcal{O}\big(\frac{\alpha\sigma^2}{\mu n_\cH}\big) + \mathcal{O}\big(\frac{\zeta^2}{\alpha^2\mu^2} + \frac{(1-\beta)^2\zeta^2}{\mu^2\beta^2} + \frac{\kappa^2\zeta^2}{(1-\sbp)^2}\big) + \mathcal{O}\big(\frac{\alpha^2\kappa^2(\sigma^2 + \beta^2\zeta_f^2/n_\cH)}{(1-\sbp)^4}\big)$: a stochastic floor matching Corollary~\ref{cor:exact_steady} to leading order, plus a Byzantine floor governed by $\zeta^2 = (\delta^{\cB}_{\max})^2\tau^2$ and the clipping-to-leak ratio, plus a leakage term proportional to $\beta^2(\zeta_f^2 + n_\cH\sigma^2)$ hidden in the bracket, which quantifies the cost of discounting honest tracking history. Analogous componentwise bounds for $\E[C^k]$ and $\E[T^k]$ follow from the second and third rows of $(I-\mathbf{A}')^{-1}b'$ exactly as in Corollary~\ref{cor:exact_steady}.
\end{corollary}
\begin{proof}
Identical cofactor computation as in Corollary~\ref{cor:exact_steady}, now with $b'_2 \neq 0$:
$[(I-\mathbf{A}')^{-1}b']_1 = \frac{1}{\det}\big(\big[(1-a'_{22})(1-a'_{33})-a'_{23}a'_{32}\big]b'_1 + a'_{12}(1-a'_{33})\,b'_2 + a'_{12}a'_{23}\,b'_3\big)$.
Bound the three terms by, respectively, $\frac{\Gamma+1}{\Gamma}\frac{b'_1}{1-a'_{11}}$ with $1-a'_{11} = \frac{\alpha\mu}{4}$ (first bracket of~\eqref{eq:leaky_steady}); $\frac{\Gamma+1}{\Gamma-1}\frac{a'_{12}b'_2}{(1-a'_{11})(1-a'_{22})}$ with $a'_{12}b'_2 = \frac{16\thp\alpha L^2\zeta^2}{\mu}$, $(1-a'_{11})(1-a'_{22}) = \frac{\alpha\mu(1-\sbp)}{8}$, and $\thp \le \frac{2}{1-\sbp}$ (the $256\kappa^2\zeta^2$ term); and $\frac{\Gamma+1}{\Gamma-1}\frac{a'_{12}a'_{23}b'_3}{(1-a'_{11})(1-a'_{22})(1-a'_{33})}$ with $a'_{12}a'_{23} = \frac{16\thp\alpha^3L^2}{\mu n_\cH}$, the product of margins at least $\frac{\alpha\mu(1-\sbp)^2}{32}$, and $\thp^2 \le \frac{4}{(1-\sbp)^2}$ (the last bracket, after factoring $3\thp n_\cH$ out of $b'_3$).
\end{proof}

\subsection{Honest Non-Clipping under Partial Isolation: an Explicit Byzantine Budget}

Under $p_b > 0$ the perturbation bound $\zeta = \delta^{\cB}_{\max}\tau$ grows with the projection radius, so the self-consistency argument of Proposition~\ref{prop:honest_nonclip} appears circular: the radius needed to cover honest fluctuations depends on the disagreement floors, which themselves grow with $\tau$. The key observation is that the dependence is \emph{affine in $\tau^2$}, so the fixed-point inequality is solvable in closed form — precisely when the aggregate Byzantine weight is below an explicit budget. Decompose $b' = b'_{(0)} + \zeta^2\, b'_{(1)}$ with
\begin{equation}
\label{eq:b_split}
b'_{(0)} = \begin{bmatrix} \dfrac{2\alpha^2\sigma^2}{n_\cH} \\[4pt] 0 \\[4pt] 36\thp n_\cH\sigma^2 + 12\thp\beta^2\zeta_f^2 \end{bmatrix},
\qquad
b'_{(1)} = \begin{bmatrix} \dfrac{9}{\alpha\mu} + \dfrac{8\alpha(1-\beta)^2}{\mu\beta^2} \\[4pt] 2\thp n_\cH \\[4pt] 3\thp n_\cH\Big(\dfrac{8(1-\beta)^2}{\beta^2} + 8L^2 + 1\Big) \end{bmatrix},
\end{equation}
so that $b' = b'_{(0)} + \zeta^2 b'_{(1)}$ with $b'_{(1)}$ independent of $\tau$ (recall $\bar s^2 = (1-\beta)^2\zeta^2/\beta^2$). Define
\begin{equation}
\label{eq:zhat_def}
\hat z_{(0)} \triangleq (I-\mathbf{A}')^{-1}\big(z^0 + b'_{(0)}\big),
\qquad
\hat z_{(1)} \triangleq (I-\mathbf{A}')^{-1} b'_{(1)},
\end{equation}
both computable from the problem parameters, and let
$N_0 \triangleq \max\{\hat z_{(0),2},\, \hat z_{(0),3}\}$, $N_1 \triangleq \max\{\hat z_{(1),2},\, \hat z_{(1),3}\}$.

\begin{proposition}[Honest non-clipping under partial isolation]
\label{prop:honest_nonclip_partial}
Suppose the assumptions and step-size/leak conditions of Theorem~\ref{thm:leaky} hold, honest agents share a common spatial initialization ($C^0=0$), and fix a horizon $K\ge1$ and confidence $\epsilon\in(0,1)$. If the aggregate Byzantine weight satisfies the budget
\begin{equation}
\label{eq:delta_budget}
\big(\delta^{\cB}_{\max}\big)^2 \;<\; \frac{\epsilon}{8(K+1)\,N_1},
\end{equation}
then every projection radius with
\begin{equation}
\label{eq:tau_partial}
\tau^2 \;\ge\; \frac{\dfrac{8(K+1)}{\epsilon}\,N_0}{\,1 - \dfrac{8(K+1)}{\epsilon}\big(\delta^{\cB}_{\max}\big)^2 N_1\,}
\end{equation}
guarantees that, with probability at least $1-\epsilon$, no honest message is modified by the projection for $k = 0,\ldots,K$. On this event the GT-PD-L trajectory coincides with the Byzantine-only dynamics, and Theorem~\ref{thm:leaky} and Corollary~\ref{cor:leaky_steady} apply to it verbatim.
\end{proposition}

\begin{proof}
Let the \emph{virtual system} run GT-PD-L in which the projection is applied to Byzantine messages only, honest messages passing unmodified; it is driven by the same oracle samples, the same per-edge uniforms realizing the masks, and the same adversary strategy (a measurable map from histories, including the adversary's own randomness, to messages) as the real system. The virtual system follows the Byzantine-only dynamics by construction, with its Byzantine rows bounded by $\zeta = \delta^{\cB}_{\max}\tau$ regardless of the adversary's messages (Proposition~\ref{prop:enforced_bound}), so Theorem~\ref{thm:leaky} applies to it unconditionally and, by the uniform bound in~\eqref{eq:leaky_limsup} and the decomposition~\eqref{eq:b_split},
\begin{equation*}
\sup_{k}\,\E[C^k_{\mathrm{ref}}] \le \hat z_{(0),2} + \zeta^2 \hat z_{(1),2} \le N_0 + \zeta^2 N_1,
\qquad
\sup_{k}\,\E[T^k_{\mathrm{ref}}] \le N_0 + \zeta^2 N_1 .
\end{equation*}
With $\delta = \frac{\epsilon}{2(K+1)}$, Markov's inequality and a union bound over the $2(K+1)$ events $\{C^k_{\mathrm{ref}} > (N_0+\zeta^2N_1)/\delta\}$, $\{T^k_{\mathrm{ref}} > (N_0+\zeta^2N_1)/\delta\}$, $k \le K$, give an event $\mathcal G$ with $\Prob(\mathcal G) \ge 1-\epsilon$ on which all honest pairwise deviations in both channels are at most
$2\sqrt{(N_0+\zeta^2N_1)\cdot 2(K+1)/\epsilon}$
for all $k \le K$. The condition~\eqref{eq:tau_partial} rearranges, using $\zeta^2 = (\delta^{\cB}_{\max})^2\tau^2$ and~\eqref{eq:delta_budget}, to exactly
$\frac{8(K+1)}{\epsilon}\big(N_0 + \zeta^2 N_1\big) \le \tau^2$,
so on $\mathcal G$ every honest pairwise deviation is at most $\tau$. The induction of Proposition~\ref{prop:honest_nonclip} then applies verbatim: if the real and virtual trajectories coincide through iteration $k \le K$, their retention probabilities and (by the common uniforms) their masks coincide, the adversary — fed the identical history — emits identical messages, and on $\mathcal G$ the projection acts as the identity on every honest message of the real system, so the updates coincide at $k+1$. Hence on $\mathcal G$ the real trajectory is the virtual one through $K$, and no honest message is clipped.
\end{proof}

\begin{remark}[The budget~\eqref{eq:delta_budget} is the standard price]
\label{rem:budget}
An explicit bound on the tolerable adversarial presence is universal in this literature: the attacked-fraction bound $\rho < 1/(1+\kappa)$ of \cite{yu2022secure} (shown there to be tight), the contraction-constant threshold of \cite{wu2023byzantine}, the $F$-local model of \cite{kuwaranancharoen2025geometric}, and the redundancy conditions of \cite{liu2021approximate} all play this role. Condition~\eqref{eq:delta_budget} is the corresponding requirement for GT-PD-L: it bounds the aggregate mixing weight that Byzantine neighbors may command, through quantities computable from the problem data. Note that it degrades with the horizon as $1/K$, an artifact of the union bound; a time-uniform version via maximal inequalities is left to future work. We also record a deterministic alternative: since the honest-clipped consensus step equals mixing by the doubly stochastic matrix $\tilde W^k$ with $[\tilde W^k]_{ij} = [\hat W^k]_{ij}c^k_{ij}$ (Remark~\ref{rem:clip_ds}), a lower bound $c^k_{ij} \ge c_{\min}$ on honest edges — equivalently, an invariant-set assumption on honest disagreement — yields $x^\top(I-\tilde W^k)x = \frac12\sum_{ij}\tilde W^k_{ij}(x_i-x_j)^2 \ge c_{\min}\,x^\top(I-\hat W^k)x$ and hence a Dirichlet-form comparison of spectral gaps, so Theorem~\ref{thm:leaky} holds with $(1-\sbp)$ replaced by $c_{\min}(1-\sbp)$ throughout, without any probabilistic argument. We prefer the probabilistic route because its hypothesis~\eqref{eq:delta_budget} is stated on the adversary, not on the trajectory. Finally, a projection radius $\tau_k$ that \emph{decays} toward the steady-state disagreement scale, in the spirit of decaying-threshold clipping defenses \cite{yu2022secure, fang2024byzantine}, would tighten the Byzantine floor $\zeta^2$ over time; the leak $(1-\beta)$ already plays the variance-reduction role that such schemes require, and we leave this extension to future work.
\end{remark}

\subsection{Selecting the Trust-Score Tolerance \texorpdfstring{$S_0$}{S0}}

The steady-state floors of Corollaries~\ref{cor:exact_steady} and~\ref{cor:leaky_steady} quantify the honest disagreement that persists under stochastic gradients. The following corollary converts them into a design rule for the threshold gate of Definition~\ref{def:data_driven_p}.

\begin{corollary}[Zero-penalty re-entry and the choice of $S_0$]
\label{cor:S0_rule}
For any honest pair $(i,j) \in \cE_\cH$ and $k\ge1$,
\begin{equation}
\begin{gathered}
\E\big[S^{k,x}_{ij}\big] \le \frac{4\,\E[C^{k-1}]}{\eta_x^2},
\qquad
\E\big[S^{k,y}_{ij}\big] \le \frac{4\,\E[T^{k-1}]}{\eta_y^2},\\
\Prob\big(p^k_{ij} < 1\big) = \Prob\big(S^k_{ij} > S_0\big) \le \frac{4}{S_0}\left(\frac{\E[C^{k-1}]}{\eta_x^2} + \frac{\E[T^{k-1}]}{\eta_y^2}\right).
\end{gathered}
\end{equation}
Consequently, with the steady-state floors $C_\infty, T_\infty$ of Corollary~\ref{cor:exact_steady} (or Corollary~\ref{cor:leaky_steady}), choosing
\begin{equation}
\label{eq:S0_rule}
S_0 \;\ge\; \frac{8}{\upsilon}\left(\frac{C_\infty}{\eta_x^2} + \frac{T_\infty}{\eta_y^2}\right)
\end{equation}
ensures that, at stationarity, each honest edge is retained with probability one except on an event of probability at most $\upsilon/2$; on the complementary event the zero-penalty regime of Lemma~\ref{lem:two_tier_spectral} is active on that edge.
\end{corollary}
\begin{proof}
For honest $i,j$: $\|x^{k-1}_j - x^{k-1}_i\|^2 \le 2\|x^{k-1}_j - \xbar^{k-1}\|^2 + 2\|x^{k-1}_i - \xbar^{k-1}\|^2 \le 4C^{k-1}$, and the denominator of~\eqref{eq:score_x} is at least $\eta_x^2$; identically for the $y$-channel. Markov's inequality on $S^k_{ij} = S^{k,x}_{ij} + S^{k,y}_{ij}$ gives the tail bound, and~\eqref{eq:S0_rule} follows by substituting the floors.
\end{proof}

\begin{remark}[The gate is necessary under stochastic gradients]
\label{rem:gate_necessary}
Since $T_\infty = \Theta\big(n_\cH\sigma^2/(1-\sbp)^2\big)$ is bounded away from zero for any fixed batch size, the honest scores $S^{k,y}_{ij}$ do \emph{not} vanish as training progresses, and an ungated rule ($S_0 = 0$, $p = e^{-\lambda S}$) keeps honest retention bounded away from one for all time. This is exactly what the retention diagnostics in Section~\ref{sec:numerics} display ($\bar p_{hh} \approx 0.13$--$0.21$ under the parameters used there). The gate~\eqref{eq:S0_rule} restores $p^k_{ij} = 1$ on honest edges once the disagreement floors are reached, activating the zero-penalty regime in finite time; shrinking the floors themselves (larger batches, smaller $\alpha$) lowers the required $S_0$ proportionally.
\end{remark}


\section{The Isolation Convergence Tradeoff}
\label{sec:tradeoff}

We now characterize the fundamental design tradeoffs in the GT-PD framework.

\subsection{Zero Penalty Isolation under Complete Dropout (\texorpdfstring{$p_b = 0$}{p\_b = 0})}

Under the two-tier scheme with $p_h = 1$ and $p_b = 0$, the algorithm yields linear convergence to a stochastic neighborhood governed solely by the native variance of the data distribution, with a consensus factor $\hat\lambda^2 = \|\hat W - J_\cH\|_2^2$ that depends only on the static honest subgraph topology and not on the Byzantine dropout probability (Lemma~\ref{lem:two_tier_spectral}). Under the threshold-gated data-driven rule, the condition $p_h = 1$ is reached in finite time once the honest disagreement falls below the tolerance (Corollary~\ref{cor:S0_rule}), and $p_b \to 0$ is achieved when the adversarial messages deviate sufficiently from the honest consensus, with discrimination strength controlled by $\lambda$. Complete isolation therefore achieves variance-optimal convergence with zero penalty on the consensus rate: by severing the adversarial edges entirely, the network sidesteps the need for spatial filtering rules that would compromise the doubly stochastic mixing property.

\subsection{Tuning the Leaky Integrator under Partial Isolation (\texorpdfstring{$p_b > 0$}{p\_b > 0})}

When edges to Byzantine agents cannot be perfectly identified and dropped, the network utilizes GT-PD-L to prevent unbounded accumulation of the tracking invariant violation. Corollary~\ref{cor:leaky_steady} makes the tradeoff governed by the leak parameter $\beta$ explicit on both sides. Increasing $\beta$ suppresses the permanent optimality penalty caused by adversarial tracking perturbations: the clipping-to-leak terms scale as $(1-\beta)^2\zeta^2/\beta^2$, so the network continuously forgets injected malicious bias. Conversely, the leak discounts \emph{honest} tracking history as well: the leakage contribution $\beta^2(\zeta_f^2 + n_\cH\sigma^2)$-type terms in~\eqref{eq:leaky_steady} grow with $\beta$, and the admissibility condition couples the step size and the leak through
$\alpha\beta \lesssim (1-\sbp)^2\mu/L^2$
in~\eqref{eq:stepsize_leaky} — a larger leak forces a proportionally smaller step size to keep the tracking row contractive. In aggressive adversarial environments (large $\zeta$), a larger $\beta$ is warranted; in benign environments a small $\beta$ permits larger step sizes and faster convergence. The two regimes meet at a $\beta$ that balances $(1-\beta)^2\zeta^2/\beta^2$ against the $\beta^2$-leakage terms, which is the quantitative version of the tradeoff.

\subsection{The Adversarial Dilemma}
\label{sec:dilemma}
The two-layer defense architecture creates a fundamental dilemma for the adversary. Consider a Byzantine agent $m \in \cB$ choosing its attack message $\tilde{x}_m^k$ at iteration $k$. The effective perturbation that agent $m$ contributes to honest agent $i$'s update is $[W^k]_{im}(\cP_{\tau,i}(\tilde{x}_m^k) - x_i^k)$, which is zero if the edge is dropped ($\xi_{im}^k = 0$) and has magnitude at most $[W^k]_{im} \cdot \tau$ if the edge is retained. The adversary faces three strategic options, none of which escapes both defense layers simultaneously.

If the adversary transmits high-magnitude messages ($\|\tilde{x}_m^k - x_i^k\| > \tau$), self-centered projection truncates the perturbation to exactly $\tau$: the adversary expends unbounded resources for the same effect as a $\tau$-bounded attack. If it transmits low-magnitude messages ($\|\tilde{x}_m^k - x_i^k\| \leq \tau$) to avoid activating the projection, the message magnitude is voluntarily bounded by $\tau$, and the dropout mechanism governs the temporal frequency of these attacks; the expected perturbation per iteration is $p_b \cdot W_{im} \cdot \tau$, which the leaky integrator contracts geometrically. If it transmits at the projection boundary, the per-message impact is maximized but still bounded by $\tau$, and the steady-state penalty is governed by the clipping-to-leak ratio $(1-\beta)^2(\delta_{\max}^{\cB})^2\tau^2/\beta^2$ of Corollary~\ref{cor:leaky_steady}, a design parameter under full control of the network operator through $\tau$ and $\beta$.

The dual-metric trust score introduces a fourth strategic dimension. A Byzantine agent $m \in \cB$ may attempt a \emph{decoupled attack}: transmitting adversarial spatial parameters $\tilde{x}_m^k$ while mimicking the honest gradient trackers (or vice versa), in the hope that anomaly in a single channel will not trigger dropout. We show that this strategy fails under the dual-metric formulation.

\begin{proposition}[Failure of decoupled attacks]
\label{prop:decoupled_failure}
Under Definition~\ref{def:data_driven_p}, suppose a Byzantine agent $m \in \cB$ communicates with an honest agent $i \in \cH$ along $(i,m) \in E$. At iteration $k$, suppose the spatial deviation satisfies $\|\tilde{x}_m^{k-1} - x_i^{k-1}\| \geq r$ for some $r > 0$, while the gradient tracker deviation is zero: $\|\tilde{y}_m^{k-1} - y_i^{k-1}\| = 0$. Then, whenever $S^{k,x}_{im} > S_0$,
\begin{equation}
\label{eq:decoupled_bound}
p_{im}^k \leq e^{\lambda S_0}\exp\!\left(-\lambda\, \frac{r^2}{\frac{1}{2}(\|x_i^{k-1}\|^2 + \|\tilde{x}_m^{k-1}\|^2) + \eta_x^2}\right).
\end{equation}
In particular, if $\|\tilde{x}_m^{k-1}\| = \Theta(\|x_i^{k-1}\|)$ and $r = \Theta(\|x_i^{k-1}\|)$, then $p_{im}^k \leq e^{-\Omega(\lambda)}$ for $S_0$ below the resulting score; and if $\|\tilde{x}_m^{k-1}\| \gg \|x_i^{k-1}\|$, the bound remains away from $1$ whenever $r / \|\tilde{x}_m^{k-1}\|$ is bounded below. The symmetric statement holds when the spatial channel shows no deviation but the gradient tracker channel deviates.
\end{proposition}

\begin{proof}
Since $\|\tilde{y}_m^{k-1} - y_i^{k-1}\| = 0$, the temporal score vanishes and $S_{im}^k = S_{im}^{k,x}$. For $S^{k,x}_{im} > S_0$, Definition~\ref{def:data_driven_p} gives $p^k_{im} = e^{-\lambda(S^{k,x}_{im}-S_0)}$, and $S^{k,x}_{im} \ge r^2/\big(\tfrac12(\|x_i^{k-1}\|^2+\|\tilde x_m^{k-1}\|^2)+\eta_x^2\big)$ yields~\eqref{eq:decoupled_bound}. The scaling claims follow by inspecting the denominator.
\end{proof}

\begin{remark}[The dual-metric dilemma]
\label{rem:dual_dilemma}
Proposition~\ref{prop:decoupled_failure} closes the gap left by a single-channel trust score: under a $y$-only score, a Byzantine agent could send arbitrary spatial parameters while mimicking honest trackers, obtaining $p_{im}^k = 1$ and bypassing the dropout layer entirely. Under the dual-metric score, the adversary must simultaneously approximate honest agents in \emph{both} channels to retain a high retention probability, at which point its per-message influence is small in both spaces and the residual perturbation is bounded by the projection radius and contracted by the leaky integrator. We stress the honest limitation already visible in~\eqref{eq:decoupled_bound}: a mimicry adversary that stays close to honest agents in both channels is \emph{not} isolated by the dropout layer — by design, no distance-based score can isolate it — and is handled by the projection and leak layers instead. The ALIE experiments in Section~\ref{sec:numerics} show precisely this division of labor.
\end{remark}

\begin{remark}[Selection of the projection radius]
\label{rem:tau_selection}
The projection radius $\tau$ serves different roles under the two isolation regimes. Under complete isolation, $\tau$ must satisfy the non-clipping condition~\eqref{eq:tau_nonclip}, which is non-circular ($\widehat C, \widehat T$ do not depend on $\tau$) and computable from problem data; $\tau$ then affects neither the rate nor the asymptotic floor. Under partial isolation, $\tau$ enters twice: the non-clipping condition~\eqref{eq:tau_partial} bounds it from below, while the Byzantine floor $\zeta^2 = (\delta^{\cB}_{\max})^2\tau^2$ in Corollary~\ref{cor:leaky_steady} penalizes it from above, and feasibility of the pair is exactly the budget~\eqref{eq:delta_budget}. Choosing $\tau$ at the lower end of~\eqref{eq:tau_partial} minimizes the Byzantine floor subject to honest safety.
\end{remark}


\section{Numerical Experiments}
\label{sec:numerics}

We evaluate GT-PD and GT-PD-L on the MNIST classification task under three Byzantine attack strategies that span the adversarial spectrum: Sign Flip (high-magnitude directional attack), A Little Is Enough (ALIE, a stealth attack that crafts perturbations within the statistical envelope of the honest gradient distribution)~\cite{baruch2019little}, and Inner Product Manipulation (IPM, which maximizes the inner product with the true gradient to induce slow divergence)~\cite{xie2020fall}. We compare against two baselines: unprotected gradient tracking (no defense) and coordinate-wise trimmed mean (CWTM)~\cite{yin2018byzantine}, a representative spatial filtering defense that breaks the doubly stochastic property of the mixing matrix.

\subsection{Experimental Setup}

The network consists of $n = 20$ agents communicating over a random regular graph of degree $4$, with $b = 4$ Byzantine agents ($20\%$ adversarial fraction). Each honest agent holds a non-IID partition of the MNIST training set generated by Dirichlet allocation with concentration parameter $\alpha_{\text{Dir}} = 0.5$, producing substantial gradient heterogeneity across agents. The model is a linear softmax classifier with explicit $\ell_2$ regularization $(\mu/2)\|\theta\|^2$ added to the cross-entropy loss, enforcing $\mu$-strong convexity as required by Assumption~\ref{as:strong_cvx}. The connectivity of the honest subgraph is verified at initialization (Assumption~\ref{as:connectivity}), and all agents share a common initial point $x_i^0 = 0$.

The GT-PD framework uses the threshold-gated dual-metric retention probability (Definition~\ref{def:data_driven_p}) with tolerance $S_0 = 3$ and sensitivity $\lambda = 1$, universal self-centered projection with radius $\tau = 1.5$ (decision messages clipped around the receiver's state $x_i^k$ and tracker messages around the receiver's tracker $y_i^k$, per Definition~\ref{def:clip}), and the leaky integrator with $\beta = 0.1$ for GT-PD-L. The step size is $\alpha = 0.05$, the regularization constant is $\mu = 0.01$, and the stochastic gradients are computed over mini-batches of size $128$. All experiments run for $30$ communication epochs. As is standard in the gradient tracking literature, the parameters exceed the conservative sufficient conditions of the theory ($\alpha$ exceeds $1/(2L)$, and the Byzantine weight in this topology exceeds the worst-case budget~\eqref{eq:delta_budget} for this horizon); the theoretical conditions are sufficient, not necessary, and the empirical convergence region is substantially larger.

\begin{figure*}[t]
\centering
\includegraphics[width=\textwidth]{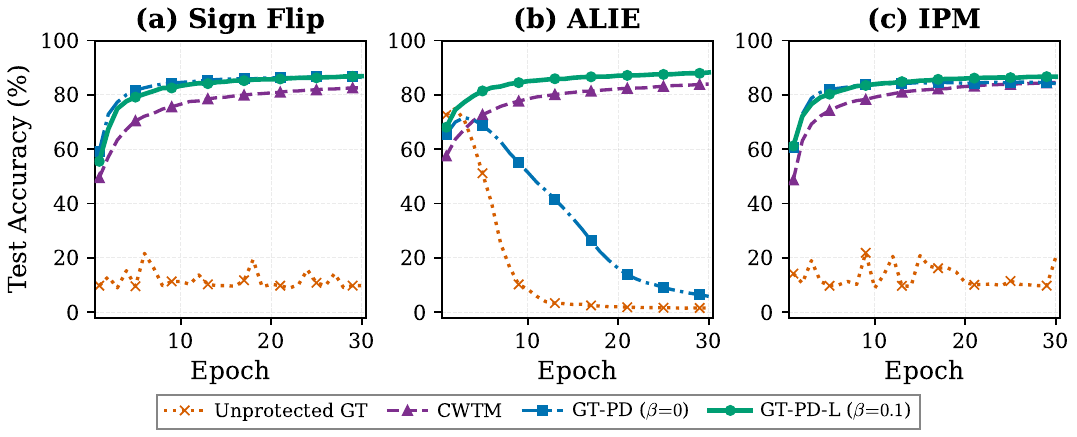}
\caption{Test accuracy of GT-PD, GT-PD-L, CWTM, and unprotected gradient tracking under three Byzantine attacks on MNIST ($n=20$, $b=4$, non-IID). GT-PD-L achieves the highest accuracy in all three settings, outperforming CWTM by $4.3$ percentage points under the stealth ALIE attack.}
\label{fig:accuracy}
\end{figure*}

\subsection{Results and Discussion}

\textbf{Sign Flipping.}
Under the sign flip attack, Byzantine agents transmit the negated and scaled honest mean, producing high-magnitude adversarial messages that are immediately clipped by the universal self-centered projection (Proposition~\ref{prop:enforced_bound}). Both GT-PD and GT-PD-L achieve $86.5\%$ and $86.9\%$ respectively, outperforming CWTM ($82.6\%$) by over four percentage points. The retention probability diagnostic confirms effective Byzantine isolation: honest-Byzantine retention averages $\bar{p}_{hb} = 0.026$ for GT-PD-L, while honest-honest retention stabilizes at $\bar{p}_{hh} = 0.13$, yielding a discrimination ratio of approximately $5{:}1$. The fact that $\bar p_{hh}$ stabilizes well below one is itself a prediction of the theory: the tracking disagreement has a stochastic floor (Corollary~\ref{cor:exact_steady}), so honest scores remain bounded away from zero, and with the tolerance $S_0 = 3$ used here the honest edges sit near the gate rather than above it (Corollary~\ref{cor:S0_rule} quantifies the tolerance that would pin them below it). The gap between GT-PD and GT-PD-L is modest in this regime because the sign flip attack is geometrically transparent and the projection alone bounds the per-message perturbation.

\textbf{A Little Is Enough (ALIE).}
The ALIE attack exposes the most significant behavioral distinction between GT-PD and GT-PD-L. GT-PD collapses to $5.9\%$ by epoch $30$, exhibiting the unbounded tracking invariant violation predicted by Remark~\ref{rem:tracking_invariant}: without the leaky integrator, the Byzantine perturbation to the gradient tracker average accumulates without contraction, causing the tracking drift $s^k$ to grow until the gradient estimates become uninformative. GT-PD-L achieves $88.3\%$, outperforming CWTM by $4.3$ percentage points. Crucially, ALIE is designed to evade distance-based detection~\cite{baruch2019little}, and the retention diagnostic confirms it: the honest-Byzantine retention $\bar{p}_{hb} = 0.33$ exceeds the honest-honest retention $\bar{p}_{hh} = 0.21$, i.e., the dropout layer does not isolate the adversary under this attack — the mimicry regime anticipated by Remark~\ref{rem:dual_dilemma}. The resilience of GT-PD-L therefore derives entirely from the interplay between the self-centered projection (bounding per-message magnitude) and the leaky integrator (geometrically contracting the accumulated tracking drift, Lemma~\ref{lem:s_bound}). This defense-in-depth property, where the framework maintains convergence even when one defense layer is defeated, is a structural advantage of the two-layer architecture described in Remark~\ref{rem:detection}.

\textbf{Inner Product Manipulation (IPM).}
Under the inner product manipulation attack, GT-PD-L achieves $86.7\%$, GT-PD achieves $84.2\%$, and CWTM achieves $84.6\%$. The retention probability diagnostic shows strong Byzantine isolation for both GT-PD and GT-PD-L ($\bar{p}_{hb} < 0.004$), indicating that the dual-metric trust score successfully detects the IPM attack in both spatial and temporal channels. The advantage of GT-PD-L over GT-PD in this regime stems from the leaky integrator preventing the slow accumulation of residual perturbations that pass through the projection during transient phases.

\textbf{The leaky integrator is necessary.} Across all three attacks, the contrast between GT-PD and GT-PD-L validates the theoretical analysis of Section~\ref{sec:convergence_leaky}. Under sign flip and IPM, both algorithms perform comparably in the early epochs, but GT-PD degrades when the tracking invariant violation accumulates beyond the contraction capacity of the consensus dynamics. Under ALIE, the degradation is catastrophic because the stealth perturbations bypass the dropout layer entirely. The leaky integrator provides a continuous decay mechanism that bounds the tracking invariant violation by $\|s^k\| \leq (1-\beta)\zeta/\beta$ (Lemma~\ref{lem:s_bound}), converting the unbounded accumulation into a bounded steady-state error.


\section{Discussion and Conclusion}
\label{sec:conclusion}

We introduced GT-PD, a Byzantine-resilient distributed optimization framework that combines universal self-centered projection with probabilistic edge dropout over a stochastic gradient tracking backbone. The convergence theory rests on two structural exactness properties — the dropout mixing matrix is doubly stochastic for every realization (Proposition~\ref{prop:ds}), and the self-centered projection preserves honest network averages exactly even when it activates on honest messages (Lemma~\ref{lem:honest_clip}) — combined with the linear-system-of-inequalities analysis of stochastic gradient tracking \cite{pu2021distributed}. Under complete Byzantine isolation ($p_b = 0$), GT-PD converges linearly to a neighborhood $\mathcal{O}\big(\frac{\alpha\sigma^2}{\mu n_\cH} + \frac{\alpha^2\kappa^2\sigma^2}{(1-\sbp)^4}\big)$ defined purely by the stochastic variance, with zero spectral penalty on the honest consensus dynamics (Theorem~\ref{thm:exact}, Lemma~\ref{lem:two_tier_spectral}). Under partial isolation ($p_b > 0$), the GT-PD-L variant introduces a continuous leaky integrator that geometrically contracts the tracking invariant violation, yielding linear convergence to a neighborhood governed by the stochastic variance and the clipping-to-leak ratio, under explicit step-size and leak conditions and an explicit budget on the aggregate Byzantine weight (Theorem~\ref{thm:leaky}, Proposition~\ref{prop:honest_nonclip_partial}); such adversarial-budget conditions are the standard price across the Byzantine optimization literature.

A central design contribution is the threshold-gated dual-metric retention probability (Definition~\ref{def:data_driven_p}), which evaluates neighbor credibility jointly in the spatial and temporal channels. The multiplicative factorization prevents adversaries from evading detection by concentrating anomalous behavior in a single channel (Proposition~\ref{prop:decoupled_failure}), while the gate makes the zero-penalty regime of Lemma~\ref{lem:two_tier_spectral} reachable in finite time under stochastic gradients — a point the steady-state analysis makes precise: tracking disagreement has an irreducible sampling floor, so an ungated score would keep honest retention below one forever, exactly as the retention diagnostics display. Corollary~\ref{cor:S0_rule} turns the steady-state floors into a principled selection rule for the tolerance $S_0$.

The numerical experiments validate the framework against three attacks spanning the adversarial spectrum. GT-PD-L achieves the highest accuracy in all three settings, outperforming the coordinate-wise trimmed mean baseline by up to $4.3$ percentage points under the stealth ALIE attack. The ALIE results reveal the defense-in-depth property: even when the dropout layer fails to isolate the adversary ($\bar{p}_{hb} > \bar{p}_{hh}$), the combination of self-centered projection and leaky integrator maintains convergence, and the contrast between GT-PD and GT-PD-L confirms the prediction of Remark~\ref{rem:tracking_invariant}.

\textbf{Limitations and future directions.}
The non-clipping guarantees are stated over a fixed horizon with a union bound, so the projection radius (and, under partial isolation, the Byzantine budget) carries a $\sqrt{K}$ (respectively $1/K$) dependence; time-uniform versions via maximal inequalities would remove it. The projection radius is fixed; an adaptive radius that contracts toward the steady-state disagreement scale could tighten the Byzantine floor $(\delta_{\max}^{\cB})^2\tau^2$, with the leak playing the variance-reduction role that decaying-threshold defenses require \cite{yu2022secure}. The tolerance $S_0$ is currently selected from worst-case steady-state bounds (Corollary~\ref{cor:S0_rule}); an adaptive threshold tracking the empirical honest disagreement would automate this choice, at the cost of a state-dependent $p_{\min}$ in the spectral bounds of Section~\ref{sec:dropout_properties}. Finally, extending the probabilistic dropout framework to time-varying directed topologies or asynchronous communication protocols would further bridge the gap between theoretical resilience guarantees and practical large-scale deployment.

\newpage


\bibliographystyle{plain}

\newpage

\appendix
\section{Appendix}
\label{sec:appendix_diagnostics}

In this appendix, we provide extended diagnostic evaluations of the GT-PD and GT-PD-L algorithms across the considered threat models. As demonstrated in Fig.~\ref{fig:supplementary}, the dual-layer defense architecture remains robust even when the data-driven dropout layer struggles to isolate the Byzantine adversaries. Specifically, during the ALIE attack, the projection and leaky integrator layers successfully guarantee convergence despite the anomalous retention probabilities. The retention panels also illustrate the stochastic disagreement floors of Corollaries~\ref{cor:exact_steady} and~\ref{cor:leaky_steady}: honest-honest retention stabilizes strictly below one, consistent with the persistent tracking disagreement predicted by the theory and with the tolerance rule of Corollary~\ref{cor:S0_rule}.

\begin{figure}[H]
\centering
\includegraphics[width=\textwidth]{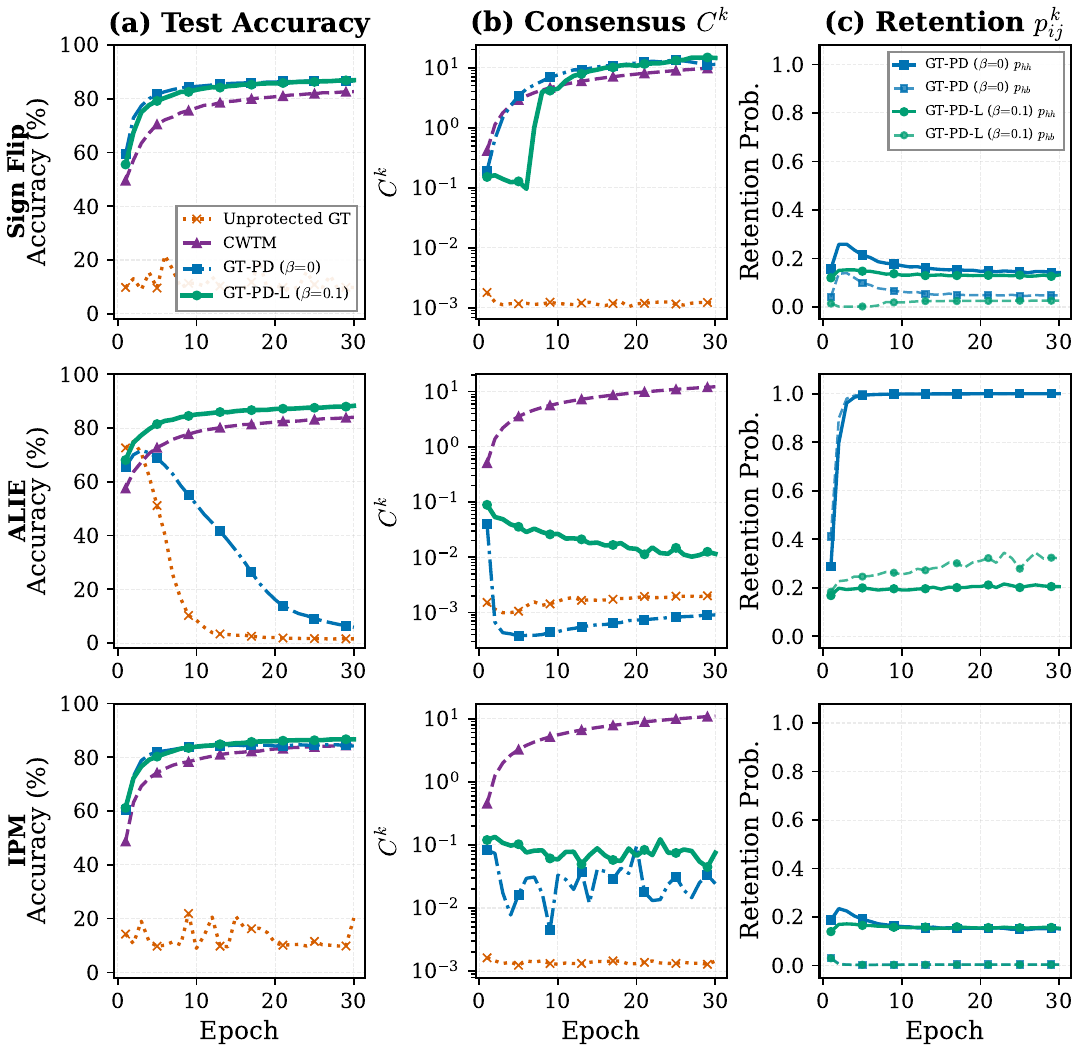}
\caption{Extended diagnostics across three attacks (rows). Column~(a): test accuracy. Column~(b): consensus disagreement $C^k$ on log scale. Column~(c): mean retention probabilities for honest-honest ($p_{hh}$, solid) and honest-Byzantine ($p_{hb}$, dashed) edges. Under ALIE, the dropout layer fails to isolate Byzantine agents ($p_{hb} > p_{hh}$), yet GT-PD-L converges via the projection and leaky integrator layers.}
\label{fig:supplementary}
\end{figure}

\end{document}